\documentclass{article}

\PassOptionsToPackage{numbers, compress, sort}{natbib}

\usepackage[preprint]{neurips_2021}

\usepackage{algorithm}
\usepackage{algorithmic}
\usepackage{amssymb}
\usepackage{amsthm}
\usepackage{adjustbox}
\usepackage[utf8]{inputenc} 
\usepackage[T1]{fontenc}    
\usepackage{hyperref}       
\usepackage{url}            
\usepackage{booktabs}       
\usepackage{amsfonts}       
\usepackage{nicefrac}       
\usepackage[protrusion=true,tracking=true,kerning=true,expansion,final]{microtype}   
\usepackage{xcolor}         
\usepackage{thm-restate}
\usepackage{amsmath}
\usepackage{dsfont}
\usepackage{graphicx}
 \usepackage{enumitem}

\renewcommand{\cite}[1]{\citep{#1}}
\definecolor{mydarkblue}{rgb}{0,0.08,0.45}
\definecolor{urlcolor}{rgb}{0,.145,.698}
\definecolor{linkcolor}{rgb}{.71,0.21,0.01}
\hypersetup{ %
	pdftitle={},
	pdfauthor={},
	pdfsubject={}, 
	pdfkeywords={},
	pdfborder=0 0 0,   
    pdfpagemode=UseOutlines,
	breaklinks=true, 
	colorlinks=true,
	bookmarksnumbered=true,
	urlcolor=urlcolor,
	linkcolor=linkcolor,
	citecolor=mydarkblue,
	filecolor=mydarkblue,
	pdfview=FitH
	}

\title{Weakly Supervised Recovery of Semantic Attributes}

\author{%
  Ameen Ali \\
  School of Computer Science\\
  Tel Aviv University\\
  Tel Aviv, Israel\\
  \texttt{ameenali023@gmail.com} \\
   \And
   Tomer Galanti \\
   School of Computer Science \\
 Tel Aviv University\\
  Tel Aviv, Israel\\
   \texttt{tomerga2@mail.tau.ac.il} \\
   \AND
   Evgenii Zheltonozhskii \\
   Department of Computer Science\\
   Technion\\
   Haifa, Israel\\
   \texttt{evgeniizh@campus.technion.ac.il} \\
   \And
   Chaim Baskin \\
   Department of Computer Science\\
   Technion\\
   Haifa, Israel\\
   \texttt{chaimbaskin@cs.technion.ac.il} \\
   \And
   Lior Wolf \\
   School of Computer Science \\
 Tel Aviv University\\
  Tel Aviv, Israel\\
   \texttt{wolf@cs.tau.ac.il} \\
} 

\begin{document}

\maketitle

\begin{abstract}
We consider the problem of the extraction of semantic attributes, supervised only with classification labels. For example, when learning to classify images of birds into species, we would like to observe the emergence of features that zoologists use to classify birds. To tackle this problem, we propose training a neural network with discrete features in the last layer, which is followed by two heads: a multi-layered perceptron (MLP) and a decision tree. Since decision trees utilize simple binary decision stumps we expect those discrete features to obtain semantic meaning. We present a theoretical analysis as well as a practical method for learning in the intersection of two hypothesis classes. Our results on multiple benchmarks show an improved ability to extract a set of features that are highly correlated with the set of unseen attributes.
\end{abstract}

\section{Introduction}

The extraction of meaningful intermediate features in classification problems is at the heart of many scientific domains. Such features allow experts to discuss a certain classification outcome and help understand the underlying structure of the problem. For example, in botany, the species are often identified based on dichotomous keys related to the shape of the leaf, the texture of the fruit, etc.~\cite{herbarium2013jepson}. In paleography, the manuscript is dated or attributed to a specific scribe, based, for example, on specific characteristics of the morphology of the letters~\cite{kups2978}. 

The attributes we discuss are binary and need to be both evidence-based and distinctive. The first property means that there is a mapping $f$, such that, every input $x$ is mapped to a vector $f(x)$ of binary values, indicating the presence of each attribute. The second property means that there should be simple rules that determine the class label $y(x)$ based on the obtained attributes $f(x)$.

Currently, deep neural networks (DNNs) are the most successful methodology for obtaining attributes from images~\cite{zhang2014panda,xu2020generating} and for classification. However, the interpretation of the DNN is a very challenging problem. In contrast, decision trees with binary decision stumps provide simple and interpretable decision rules based on attributes~\cite{guidotti2018survey}. 
We propose to merge two approaches by training DNN to produce quantized representations which are suitable for classification both by an MLP and by a decision tree. 

To perform this hybrid learning task, we provide a theoretical analysis of the problem of learning at the intersection of two hypothesis classes. We study the optimization dynamics of training two hypotheses from two different classes of functions. The first one is trained to minimize its distance from the second one, and the second one is trained to minimize its distance from the first one and the target function. We call this process {\em Intersection Regularization}, as it regularizes the second hypothesis to be close to the first hypothesis class. We theoretically discuss the conditions on the loss surface, for which this process converges to an equilibrium point or local minima. 

The method that we develop for concurrently training a network and a tree is based on the proposed analysis. We learn a quantized representation of the data and two classifiers on top of this representation: a decision tree and MLP. The two classifiers are trained using intersection regularization. Since the human-defined attributes are usually sparse, we also apply $L_1$ regularization on the quantized vector of activations.

In an extensive set of experiments, we demonstrate that discrete representations along with decision stumps learned with the proposed method are highly correlated with a set of unseen human-defined attributes. At the same time, the overall classification accuracy is only slightly reduced compared to standard cross-entropy training.  

The central contributions in this paper are as follows:
    (i) We identify a novel learning problem that we call {\em Weakly Unsupervised Recovery of Semantic Attributes}. In this setting, the samples are associated with abstract binary attributes and are labeled by class membership. The algorithm is provided with labeled samples and is judged by the ability to recover the binary attributes without any access to them. We provide concrete measures for testing the success of a given method for this task. (ii) We introduce a method for recovering the semantic attributes. This method is based on a novel regularization process for regularizing a neural network using a decision tree. We call this regularization method, {\em `Intersection Regularization'}, and (iii) We theoretically study the convergence guarantees of the new regularization method to equilibrium points and local minima.

\section{Related Work}

We list various approaches for learning explainable models in the literature and discuss their similarities and dissimilarities with our approach.

{\bf Interpretability\quad} Developing tools and techniques to better interpret existing deep learning based approaches and to that end building explainable machine learning algorithms, is a fast-growing field of research. In computer vision, most contributions are concerned with providing an output relevance map. These methods include saliency-based methods~\cite{simonyan2013deep,zeiler2014visualizing,mahendran2016visualizing,zhou2016learning,dabkowski2017real,zhou2018interpreting,gur2020visualization}, {Activation Maximization}~\cite{erhan2009visualizing} and Excitation Backprop~\cite{zhang2018top}, perturbation-based methods~\cite{fong2017interpretable,fong2019understanding}. Shapley-value-based methods~\cite{lundberg2017unified} enjoy theoretical justification and the Deep Taylor Decomposition~\cite{montavon2017explaining} provides a derivation that is also applicable to Layer-wise Relevance Propagation (LRP)~\cite{bach2015pixel} and its variants~\citep{gu2018understanding,iwana2019explaining,nam2019relative} presented RAP. Gradient based methods are based on the gradient with respect to the layer's input feature map and include  Gradient*Input~\citep{shrikumar2016not}, Integrated Gradients~\citep{sundararajan2017axiomatic}, 
SmoothGrad~\citep{smilkov2017smoothgrad}, FullGrad~\citep{srinivas2019full} and Grad-CAM ~\citep{selvaraju2017grad}. 

Methods that provide an output relevance map suffer from several notable disadvantages. First, many of these methods were shown to suffer from a bias toward image edges and fail sanity checks that link their outcome to the classifier~\cite{asano2019critical}. {Furthermore, even though these methods are useful for visualization and downstream tasks, such as weakly supervised segmentation~\citep{li2018tell}, it is not obvious how to translate the image maps produced by these methods into semantic attributes, i.e., extract meaning from a visual depiction.} 
{An additional disadvantage of this approach is that it does not provide a direct method for evaluating whether a given neural network is interpretable or not and the evaluation is often done with related tasks, such as segmentation or on measuring the classifier's sensitivity to regions or pixels that were deemed as being important to the classification outcome or not. In our framework, we suggest objective measures in which one can quantify the degree of interpretability of a given model.}

Since linear models are intuitively considered interpretable, a framework in which the learned features are monotonic and additive has been developed~\cite{alvarez2018}. The explanation takes the form of presenting each attribute's contribution while explaining the attributes using prototype samples. Unlike our framework, no attempt was made to validate that the obtained attributes correspond to a predefined list of semantic attributes.

Local interpretability models, such as LIME~\cite{Ribeiro:2016:WIT:2939672.2939778}, approximate the decision surface for each specific decision by a linear model and are not aimed at extracting meaningful attributes as in our method. Besides, such models are known to be sensitive to small perturbations of the input~\cite{alvarez2018robustness,yeh2019sensitive}. 


{\bf Fine-Grained Classification\quad} Fine-grained classification aims at differentiating subordinate classes of a common superior class. Those subordinate classes are usually defined by domain experts, based on complicated rules, which typically focus on subtle differences in particular regions. Because of the inherent difficulty in classifying slightly different classes, many contributions in this area often aim at detecting informative regions in the input images that aid in classifying them~\citep{DBLP:conf/cvpr/DuanPCG12,yang2018learning,Chen2019ThisLL,hu2019better,huang2020interpretable,zhuang2020learning}. 
In our work, we would like to recover semantic attributes that are not necessarily associated with specific image regions, e.g., mammal/reptile, omnivore/carnivore. 

{{\bf Disentanglement}\quad A disentangled representation~\cite{10.1109/TPAMI.2013.50,Peters2017,10.1162/neco.1992.4.6.863,lake_ullman_tenenbaum_gershman_2017,47658} is a representation which contains multiple independent parts. Various methods, such as~\cite{marx2019disentangling,10.5555/3327345.3327490,NEURIPS2020_00a03ec6,brakel2017learning} have been able to effectively learn disentangled representations from data. This setting is different than ours in two ways. First, the attributes that we would like to recover are not independent. Second, in disentanglement, we are typically interested in recovering any set of features that represent the data in a disentangled manner. In our case, the success is measured vs. a specific set of attributes.}

{\bf Attributes Recovery}\quad Attributes recovery has been an active research direction in computer vision. Attributes are typically referred as human-interpretable features that can describe the input image or the class it belongs to. In~\cite{10.1007/978-3-030-11015-4_1} they cast the problem of attributes recovery as a supervised learning task. In~\cite{Yu_2013_CVPR,5206594,5206772,Sattar_2017_ICCV} they focus on category-level attributes, in which one would like to learn a set of attributes that describe the classes of the different images in the dataset. In several cases and especially in fine-grained classification, predicting the class of an image is a more nuanced process, where the values of attributes may vary between images in the same category. In our paper, we focus on recovering instance-level attributes. 

{\bf Explainability of Decision Trees}\quad
Decision trees are known to be naturally explainable models, as long as the decision rules are easy to interpret. However, as far as we know, it was not shown before that learning trees, while optimizing the features leads to the emergence of attributes of a large set of relevant attributes. 

In the context of recommendation systems interpretability, meta-trees were used for providing per-user decision rules~\cite{10.1145/3375627.3375876}. This method relies on a pre-existing set of features rather than extracting them from the input. Another line of work~\cite{frosst2017distilling}, which also employs preexisting features, distills the information in a deep network into a soft decision tree~\cite{irsoy2012soft}. Unlike our work, the neural network is not optimized to produce suitable features. 

More relevant to our work is the Adaptive Neural Trees method~\cite{pmlr-v97-tanno19a}, which also combines decision trees and neural networks. In this method, trees of dynamic architectures that include network layers are grown so that the underlying network features gradually evolve. It is shown that the learned tree divides the classes along meaningful axes but the emergence of semantic attributes was not shown.


{\bf Quantization\quad} Quantization is the conversion of the floating-point data to low-bit discrete representation.  The most common approaches for training a quantized neural network employ two sets of weights~\cite{hubara2016quantized,zhou2016dorefa}. The forward pass is performed with quantized weights, and updates are performed on full precision ones, i.e., approximating gradients with the straight-through estimator (STE) \cite{bengio2013estimating}. In this work, we make use of quantizations in order to learn discrete, interpretable representations of the data within a neural network.


\begin{figure}[t]
\centering
\includegraphics[width=1\textwidth]{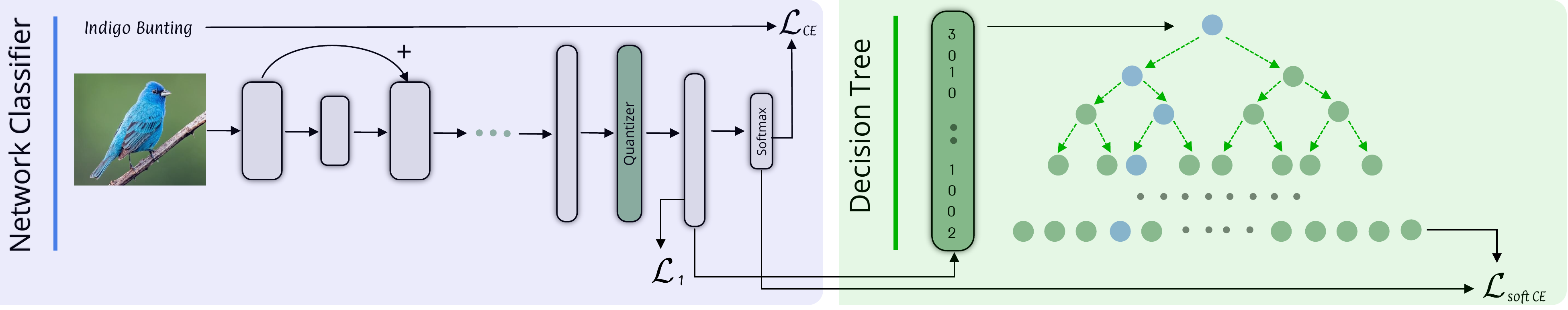}
\caption{{\bf An illustration of our method.} The model has three main components: {\bf(a)} a quantized representation function $F_q$, {\bf(b)} a classifier neural network $G$ and {\bf(c)} a multivariate decision tree $T$. We have three losses: a $L_1$ regularization over the masked quantized features, a cross entropy loss between the network and the ground-truth labels and a soft cross entropy loss between the outputs of the tree and the network.  }
\label{figures:model}
\end{figure}







\section{Problem Setup}\label{sec:setup}

In this section, we formulate a new learning setting that we call {\em Weakly Unsupervised Recovery of Semantic Attributes}. In this setting, there is an unknown target function $y:\mathcal{X} \to \mathcal{Y}$ along with an attributes function $f:\mathcal{X} \to \mathcal{U}^{d}$ that we would like to learn. Here, $\mathcal{X} \subset \mathbb{R}^n$ is a set of instances and $\mathcal{Y}$ is a set of labels. The function $f$ is assumed to be $d_0$-sparse, i.e., $\forall x \in \mathcal{X}:~\|f(x)\|_0 \leq d_0$. The values of $f$ are taken from a latent space $\mathcal{U}^d$, where $\mathcal{U} = \mathbb{R}$, $[-1,1]$ or $\{\pm 1\}$, depending on the task.

For example, $\mathcal{X}$ can be a set of images of $k$ animal species and $\mathcal{Y} = \{\textnormal{e}_i\}^{k}_{i=1}$ is the set of labels, where $\textnormal{e}_i \in \mathbb{R}^k$ is the $i$'th elementary vector. In this case, $y$ is a classifier that takes an animal image $x$ and returns the species of the animal. The function $f$ corresponds to a set of $d$ binary attributes $\{f_i\}^{d}_{i=1}$ (e.g., carnivore/herbivore). Each attribute $f_i$ takes an image $x$ and returns whether the animal illustrated in $x$ satisfies $i^{\text{th}}$ attribute.


Similar to the standard learning setting, the set $\mathcal{X}$ is endowed with a distribution $D$, from which samples $x$ are being taken. The learning algorithm is provided with a set $\mathcal{S} = \{(x_i,y(x_i))\}^{m}_{i=1}$ of $m$ i.i.d.\ samples. In the empirical risk minimization setting \cite{vapnik1992erm}, the learning algorithm is supposed to fit a hypothesis $H:\mathcal{X} \to \mathcal{Y}$ from a given {\em hypothesis class} $\mathcal{H}$ that would minimize the {\em expected risk}:
\begin{equation}\label{eq:expected}
\mathcal{L}_D[H,y] := \mathbb{E}_{x \sim D} [\ell(H(x),y(x))],
\end{equation}
where $\ell:\mathcal{Y}^2 \to [0,\infty)$ is some loss function, for example, the $\ell_2$ loss function $\ell_2(a,b) = \|a-b\|^2_2$ for regression or the cross-entropy loss for classification. Since the algorithm is not provided with a full access to $D$, the algorithm minimizes the empirical version of it:
\begin{equation}
\mathcal{L}_{\mathcal{S}}[H,y] := \frac{1}{m} \sum_{(x,y(x)) \in \mathcal{S}}\ell(H(x),y(x)),
\end{equation}
or by minimizing $\mathcal{L}_{\mathcal{S}}[H,y]$ along with additional regularization terms.

In contrast to the standard learning setting, in our framework we would like to learn a model $H = G \circ q \circ F$ that minimizes the expected risk, but also recovers the semantic attributes $f(x)$ (without any access to these attributes). Here, $F$ and $G$ are a trainable representation function and classifier and $q$ is a pre-defined discretization operator.

Put differently, the labels are given as proxy labels, while the actual goal is to learn a discrete-valued representation $q(F(x))$ of the data that maximizes some measure of feature fidelity $d_D(F) = d_{D}(f; b \circ q\circ F)$ that depends on the distribution $D$ and a binarization function $b$ that translates a vector into a binary vector (see Sec.~\ref{sec:method} for details). 

\subsection{Measures of Feature Fidelity}
\label{sec:metrics}

The goal of training in the proposed setting is to learn a representation of the data $q(F(x))$ that maximizes the fidelity of the extracted features with respect to an unseen set of ground-truth binary attributes $f(x)$. 

In this section, we define a generic family of functions $d_{D}(f;g)$ for measuring the fidelity of a multi-variate function $g:\mathcal{X} \to \mathcal{U}^{d}$ with respect to a multi-variate function $f:\mathcal{X} \to \mathcal{U}^{n}$  over a distribution $D$ of samples. 

Let $r(q_1,q_2;D)$ be a measure of accuracy between two univariate functions $q_1,q_2:\mathcal{X} \to \mathcal{U}$ over a distribution $D$. In this paper, since we consider imbalanced attributes, $r$ is the F1 score. 
Finally, we extend $r$ to be annotation invariant by using $\hat{r}(q_1,q_2;D) := \max \{r(q_1,q_2;D), r(q_1,1-q_2;D)\}$ as a measure of accuracy which is invariant to whether positive samples are denoted by $1$ or $0$. 


Let $g:\mathcal{X} \to \{0,1\}^d$ and $f:\mathcal{X} \to \{0,1\}^n$ be two multivariate binary functions. We denote:
\begin{equation}\label{eq:metric}
\begin{aligned}
d_{D}(f \| g) := \max_{\pi: [n] \to [d]} \frac{1}{n} \sum^{n}_{i=1} \hat{r}(f_i,g_{\pi(i)};D) = \frac{1}{n} \sum^{n}_{i=1} \max_{j \in [d]} \hat{r}(f_i,g_{j};D)
\end{aligned}
\end{equation}
This quantity measures the average similarity between each feature $f_i$ in $f$ to some feature $g_{j}$ in $g$. 

{The fidelity of $g$ with respect to $f$ is the harmonic mean $d_D(f;g) := 2\frac{d_D(f\|g)\cdot d_D(g\| f)}{d_D(f\|g) + d_D(g\| f)}$, which is a symmetric measure of similarity. Informally, $d_D(f\|g)$ measures the extent at which the set of attributes in $f$ can be treated as a subset of the set of attributes in $g$ and $d_D(f;g)$ as the extent at which the sets of attributes in $f$ and $g$ are equivalent. We use the harmonic mean since it penalizes low values in contrast to the arithmetic mean.}

When measuring the fidelity over a finite set of test samples $\mathcal{S}_T$, the proposed $d_{D}(f ; g)$ considers the discrete uniform distribution $D = U[\mathcal{S}_T]$. 

Finally, in some cases, it is more relevant to measure the distance between $f$ and $g$, instead of their similarity. We note that our method can be readily extended to this case by taking $r$ to be a distance function and replacing the maximization in Eq.~\eqref{eq:metric} and the definition of $\hat{r}$ by minimization. For example, for real-valued functions (i.e., $\mathcal{U} \subset \mathbb{R}$), it is reasonable to use the distance between the two functions, $r(q_1,q_2;D) = \mathbb{E}_{x\sim D}[\vert q_1(x) - q_2(x)\vert]$.

\section{Intersection Regularization}

A key component of our method in Sec.~\ref{sec:method} is the proposed notion of {\em Intersection Regularization}. In this section, we introduce and theoretically study the optimization dynamics of the intersection regularization.

Suppose we have two hypothesis classes $\mathcal{G} = \{G_{\theta} \mid \theta \in \Theta\}$ and $\mathcal{T} = \{T_{\omega}\mid \omega \in \Omega\}$, where $\Theta$ and $\Omega$ are two sets of parameters. Intersection regularization involves solving the following problem:
\begin{equation}
\begin{aligned}
&\min_{\theta \in \Theta} \min_{\omega \in \Omega} \mathcal{Q}(\theta,\omega) \\
\textnormal{where: }&\mathcal{Q}(\theta,\omega) :=  \mathcal{L}_{\mathcal{S}}[G_{\theta},y] + \mathcal{L}_{\mathcal{S}}[G_{\theta},T_{\omega}]
\end{aligned}
\end{equation}
In this problem, we are interested in learning a hypothesis $G_{\theta} \in \mathcal{G}$ that is closest to the target function $y$ among all members of $\mathcal{G}$ that can be approximated by a hypothesis $T_{\omega} \in \mathcal{T}$. We can think of $\mathcal{T}$ as a prior knowledge we have on the target function $y$. Therefore, in some sense, the term $\mathcal{L}_{\mathcal{S}}[G_{\theta},T_{\omega}]$ acts as a regularization term that restricts $G_{\theta}$ to be close to the class $\mathcal{T}$. 

In Sec.~\ref{sec:method}, we use intersection regularization to train a neural network $G$ over a quantized representation that mimics a decision tree $T$ and minimizes the classification error. In this case, the class $\mathcal{T}$ consists of decision trees of a limited depth ($\omega$ is a vector that allocates an encoding of the tree, including its structure and {decision rules}) and $\mathcal{G}$ is a class of neural networks of a fixed architecture ($\theta$ is a vector of the weights and biases of a given network). The underlying quantized representation thus obtained is suitable for classification by both a decision tree and a neural network. 

The following analysis focuses on two main properties: finding a local minima of $\mathcal{Q}$ and arriving at an equilibrium point of $\mathcal{Q}$. An equilibrium point of $\mathcal{Q}$ is a pair $(\theta,\omega)$, such that, $\mathcal{Q}(\hat{\theta},\hat{\omega}) 
= \min_{\theta} \mathcal{Q}(\theta,\hat{\omega}) = \min_{\omega} \mathcal{Q}(\hat{\theta},\omega)$. The proofs are given in the supplementary material and are based on Thm.~2.1.14 by \citet{10.5555/2670022} and the analysis by \citet{song2017block}.





The following proposition shows that under certain conditions, it is possible to converge to an equilibrium of $\mathcal{Q}(\theta,\omega)$ when iteratively optimizing $\theta$ and $\omega$. To show that, we assume that $\mathcal{Q}$ is a convex function with respect to $\theta$ for any fixed value of $\omega$. {This is true, for example, when $G_{\theta}$ is the linearization of a wide neural network~\citep{NEURIPS2019_0d1a9651}, which also serve as universal approximators~\citep{Ji2020Neural}. It has also been proven that the optimization dynamics of wide neural networks match the dynamics of their linearized version~\citep{NEURIPS2019_0d1a9651}.} 
In addition, we assume that one is able to compute a global minimizer $\omega$ of $\mathcal{L}_{\mathcal{S}}[G_{\theta},T_{\omega}]$ for any $\theta$. This is typically impossible, however, it is reasonable to assume that one is able to approximately minimize $\mathcal{L}_{\mathcal{S}}[G_{\theta},T_{\omega}]$ with respect to $\omega$ if it is being optimized by a descent optimizer. Throughout the analysis, we assume that $\cup_{\omega \in \Omega} \arg\min_{\theta} \mathcal{Q}(\theta,\omega)$ is well-defined and bounded and that $\lim_{\theta:~\|\theta\| \to \infty}\mathcal{L}_{\mathcal{S}}[G_{\theta},y] = \infty$.

\begin{restatable}{prop}{lemOne}\label{lem:1}
Assume that $\mathcal{Q}(\theta,\omega)$ is convex and $\beta$-smooth w.r.t $\theta$ for any fixed value of $\omega$. Let $\theta_1$ be some initialization and $\omega_1  \in \arg\min_{\omega} \mathcal{Q}(\theta_1,\omega)$. We define $\theta_t$ to be the weights produced after $t$ iterations of applying Gradient Descent on $\mathcal{Q}(\theta,\omega_{t-1})$ over $\theta$ with learning rate $\mu < \beta^{-1}$ and $\omega_t = \arg\min_{\omega} \mathcal{Q}(\theta_{t-1},\omega)$. Then,  we have:
\begin{equation*}
\begin{aligned}
\lim_{t \to \infty} \mathcal{Q}(\theta_{t},\omega_{t}) 
= \lim_{t \to \infty} \min_{\theta} \mathcal{Q}(\theta,\omega_{t}) = \lim_{t \to \infty} \min_{\omega} \mathcal{Q}(\theta_{t},\omega)
\end{aligned}
\end{equation*}
\end{restatable}

The following proposition shows that if we apply Block Coordinate Gradient Descent (BCGD) to optimize $\theta$ and $\omega$ for minimizing $\mathcal{Q}(\theta,\omega)$ (starting at $(\theta_1,\omega_1)$), then, they converge to a local minima that is also an equilibrium point. The BCGD iteratively updates: $\theta_{t+1} = \theta_t - \mu \nabla \mathcal{Q}(\theta_{t},\omega_{t})$ and $\omega_{t+1} = \omega_t - \mu \nabla \mathcal{Q}(\theta_{t+1},\omega_{t})$. Throughout the analysis we assume that the sets $\cup_{\omega \in \Omega} \arg\min_{\theta} \mathcal{Q}(\theta,\omega)$ and $\cup_{\theta \in \Theta}\arg\min_{\omega} \mathcal{Q}(\theta,\omega)$ are well-defined and bounded. 

\begin{restatable}{prop}{lemTwo}\label{lem:2}
Assume that $\mathcal{Q}(\theta,\omega)$ is a twice-continuously differentiable, element-wise convex (i.e., convex w.r.t $\theta$ for any fixed value of $\omega$ and vice versa), Lipschitz continuous and $\beta$-smooth function, whose saddle points are strict. Let $\theta_t$, $\omega_t$ be the weights produced after $t$ iterations of applying BCGD on $\mathcal{Q}(\theta,\omega)$ with learning rate $\mu < \beta^{-1}$. $(\theta_t,\omega_t)$ then converges to a local minima $(\hat{\theta},\hat{\omega})$ of $\mathcal{Q}$ that is also an equilibrium point.
\end{restatable}

\section{Method}\label{sec:method}

In this section, we discuss propose our method for dealing with the problem introduced in Sec.~\ref{sec:setup}.

{\bf Model\quad} Informally, our algorithm aims at learning a sparse discrete representation $F_q(x) := q(F(x))$ of the data, that is suitable for classification by a decision tree $T$ of small depth. {Intuitively, this leads to a representation that supports classification by relatively simple decision rules.} For this purpose, we consider a model of the following form: $H_{tree} = T \circ F_q$, where $F:\mathcal{X} \to \mathbb{R}^{d}$ is a trainable neural network and $T:\mathbb{R}^{d} \to \mathbb{R}^k$ is a decision tree from a class $\mathcal{T}$ of multivariate regression decision trees of maximal depth $d_{\max}$. The function $G$ (and $T$) is translated into a classifier by taking $\arg\max_{i \in [k]} G(F_q(x))$ (and $\arg\max_{i \in [k]} T(F_q(x))$).


\begin{algorithm}[t]
 \caption{Intersection Regularization based Sparse Attributes Recovery}
 \label{alg:method}
 \begin{algorithmic}[1]
 \REQUIRE $\mathcal{S} = \{(x_i,y(x_i))\}^{m}_{i=1}$ - dataset; $\lambda_1,\lambda_2,\lambda_3$ - non-negative coefficients; $I$ - number of epochs; $s$ - batch size; $A$ - Tree training algorithm (and splitting criteria);
 \STATE Initialize $F,G$ and $T=\textnormal{None}$;
 \STATE Partition $\mathcal{S}$ into batches $Batches(\mathcal{S})$ of size $s$;
 \FOR{$i = 1,\dots,I$}
 \STATE $\bar{\lambda}_2 = \mathds{1}[i>1] \cdot \lambda_2$;
 \STATE $\mathcal{S}'=\emptyset$;
 \FOR{$B \in Batches(\mathcal{S})$}
 \STATE Update $G$ using GD to minimize $\lambda_1 \mathcal{L}_{B}[G \circ F_q,y] + \bar{\lambda}_2 \mathcal{L}_{B}[G \circ F_q, T \circ F_q]$;
 \STATE Update $F$ using GD to minimize $\lambda_1 \mathcal{L}_{B}[G \circ F_q,y] + \bar{\lambda}_2 \mathcal{L}_{B}[G \circ F_q, T \circ F_q] + \lambda_3\mathcal{R}_{B}[F]$; 
 \STATE Extend $\mathcal{S}' = \mathcal{S}' \cup \{(F_q(x),G(F_q(x))) \mid x \in B\}$;
 \ENDFOR
 \STATE Initialize decision tree $T$;
 \STATE Train $T$ over $\mathcal{S}'$ using $A$;
 \ENDFOR
 \STATE {\bf return} $F,G,T$;
  \end{algorithmic}
\end{algorithm}

In order to learn a neural network $G$ that minimizes the classification error and also approximates a decision tree, we apply intersection regularization between a class of neural networks $\mathcal{G}$ and the class of decision trees $\mathcal{T}$. Our objective function is decomposed into several loss functions. For each loss, we specify in brackets $``[~\cdot~]''$ the components that are responsible for minimizing the specified loss functions. For a full description of our method, see Alg.~\ref{alg:method}. 

For each epoch (line 3), we iteratively update the network to minimize its objective function using GD and train the decision tree from scratch. To optimize the neural network, we use two loss functions (lines 7-8). The first one is the cross-entropy loss of $H_{net}$ with respect to the ground-truth labels, and the second one is the soft cross-entropy loss of $H_{net}$ with respect to the probabilities of $H_{tree}$,
\begin{align}
&\mathcal{L}_{B}[G \circ F_q,y] \textnormal{ and } \mathcal{L}_{B}[G \circ F_q, T \circ F_q],
\end{align}
where the loss function $\ell:\Delta_k \times \Delta_k \to [0,\infty)$ is the cross-entropy loss that is defined as follows: $\ell(u,v) := -\sum^{k}_{i=1} v_i \cdot \log(u_i)$, $k$ is the number of classes and $\Delta_k$ is the standard simplex. {The second loss is applied only from the second epoch onward (line 4).} 


To encourage sparse representations~\citep{tibshirani2996lasso,koh07l1}, we also apply $L_1$ regularization over the quantized representation of the data. 
\begin{equation}\label{eq:l_one}
\mathcal{R}_{B}[F] := \frac{1}{m} \sum^{s}_{i=1} \|F_q(x_i)\|_1
\end{equation}

{\bf Tree Optimization\quad} During each epoch, we accumulate a dataset $\mathcal{S}'$ of pairs $(F_q(x),G(F_q(x)))$ for all of the samples $x$ that have been incurred until the current iteration (line 9 in Alg.~\ref{alg:method}). By the end of the epoch, we train a multivariate regression decision tree $T$ from scratch over the dataset $\mathcal{S}'$ (see line 12 in Alg.~\ref{alg:method}). To train the decision tree, we used the CART algorithm~\cite{loh2011classification} with the information gain splitting criteria.

{\bf Quantization and Binarization\quad} The function $F:\mathcal{X} \to \mathbb{R}^n$ is a real-valued multivariate function. To obtain a discrete representation of the data, we discretize the outputs of $F$ using the uniform quantizer~\citep{sheppard1897calculation,vanhoucke2011cpu} $q$, see supplementary for a step-by-step listing of the algorithm. 
This quantization employs a finite number of equally-sized bins. Their size is calculated by dividing the input range into $2^r$ bins, where $r$ specifies the number of bits for encoding each bin.

Since the round function is non-differentiable, the gradients of the uniform quantization are usually approximated~\cite{bengio2013estimating,ramapuram2019improving,yang2019quantizationnet}. In this paper, we utilize the straight-through estimator (STE)~\citep{bengio2013estimating}, which assumes that the Jacobian of rounding is just identity, to estimate gradients of the discretization. 


{For applying the feature fidelity measure proposed in Sec.~\ref{sec:metrics}, we first cast the discrete vector $q(F(x))$ into a vector of binary features $b(q(F(x)))$ and compute $d_D(F) := d_D(f ; b \circ q \circ F)$, where $b$ is a binarization function. In this paper, we use the following binarization scheme. For a given vector $v = q(F(x))$ we compute $u = b(v) := (u^1\|u^2)$ as follows: for all $i \in [n]$ and $j \in [2^{r}]$, we have: $u^1_{i,j} = \mathds{1}[v_i=j]$ and $u^2_{i,j} = \mathds{1}[v_i\neq j]$, where $\mathds{1}$  is the indicator function. The dimension of $u$ is $2^{r+1} \cdot n$.}

\section{Experiments}\label{sec:experiments}

In this section, we evaluate our method on several datasets with attributes in comparison with various baselines. To evaluate the attributes recovery quality, we use the evaluation metric explained in Sec.~\ref{sec:metrics} based on the F1 score.

{\bf Implementation Details\quad}  The architecture of the feature extractor $F$ is taken from~\citep{du2020fine} with the published hyperparameters and is based on ResNet-50. The classifier $G$ is a two-layered fully connected neural network. We initialize the ResNet with pre-trained weights trained on ILSVRC2012~\citep{10.1007/s11263-015-0816-y}. {When training our method and the baselines we employed the following early stopping criteria: we train the model until the second epoch for which the accuracy rate drops over a validation set, and report the results on the last epoch before the drop.}

{Our method was run on 4 GeForce RTX 2080 Ti GPUs. On the aPscal and aYahoo datasets, each epoch takes about 2-3 minutes, on the AwA dataset each epoch takes 17-25 minutes, on CUB-200-2011 dataset each epoch takes 8-15 minutes. The range of time depends on the representation capacity (e.g., no. bits, dimension).}

Throughout the experiments, we used the following default hyperparameters, except in our ablation studies, where we varied the hyperparameters to evaluate their effect. The coefficients for the loss functions $\lambda_1 = 2$, $\lambda_2 = 1$, $\lambda_3 = 0.001$. 
{Optimization was carried out using SGD.}

{\bf Datasets\quad} Throughout our experiments, we used the following datasets. (i) The aYahoo dataset~\citep{farhadi2009describing}, consisting of images from $12$ classes (e.g., bag, goat, mug). Each image is labeled with $64$ binary attributes. The images are collected from Yahoo. (ii) aPascal dataset~\citep{farhadi2009describing} consisting of $20$ classes of images. The images are labeled with the same $64$ binary attributes as in the aYahoo dataset. (iii) Animals with attributes dataset (AwA2)~\citep{xian2018zero} consisting of images from $50$ animal classes, each class is labeled with $85$ numeric attribute values, (iv) CUB-200-2011~\citep{WahCUB_200_2011} consisting of $200$ bird species classes, each image is labeled with $312$ binary attribute values. {We used the standard train/evaluation splits of the datasets.}

{\bf Baseline methods\quad} We compare our method with various methods that capture a wide variety of approaches: (i) SDT~\citep{frosst2017distilling} and ANT~\citep{pmlr-v97-tanno19a}, train a neural network of a tree architecture with the intention of learning high-level concepts. (ii) WS-DAN~\citep{hu2019better} is a method that generates attention maps to represent the object’s discriminative parts in order to improve the classification. (iii) DFL-CNN and Nts~\citep{wang2018learning,yang2018learning} use a network that captures class-specific discriminative regions. Region Grouping~\citep{huang2020interpretable} is a similar method that also uses a regularization term enforcing the empirical distribution of part occurrence to align a U-shaped prior distribution. (iv) ProtoPNet~\citep{Chen2019ThisLL} computes similarity scores of informative patches in the image with learned prototype images. These similarity scores are then aggregated by an MLP classifier. All of the baselines, except ProtoPNet, are provided with labeled samples, without access to ground-truth semantic segmentation of any kind. In ProtoPNet they make use of cropping-based preprocessing that makes use of the bounding boxes provided with the CUB-200-2011 dataset. 

{To measure the baselines' and our method's performance on our task, we apply quantization on the penultimate layer of each model. For each method, we apply configurations with penultimate layer dimension $128,256$ or $512$ and apply uniform quantization with a $1,2,3$ or $4$ bits representation. We report, for each method, the results of the configuration that provides the highest accuracy rate on the validation set. Note that this selection is based on accuracy and not on the feature fidelity score.} The fidelity score is computed as follows: $d_D(F) := d_D(f ; b \circ q \circ F)$, where $F$ is the penultimate layer of the model and $q$, $b$ are the uniform quantization and binarization operators. 

\begin{table}[t]
\caption{{\bf Comparing the performance of various baselines with our method.} We report the classification accuracy rate (Acc) and the feature fidelity score ($d_D(F)$) for each method on each one of the datasets. {For ProtoPNet on CUB-200-2011, we report the results with (right) and without (left) using cropping-based augmentations.} As can be seen, our method outperforms the other methods in terms of recovering the semantic attributes across datasets.} 
\centering
\resizebox{\linewidth}{!}{
\begin{tabular}{l@{~}c@{~}c@{~~}c@{~~}c@{~~}c@{~}c@{~~}c@{~~}c@{~~}c@{~}c@{~~}c@{~~}c@{~~}c@{~}c@{~~}c@{~~}c@{~~}c@{~~}c@{~~}c@{~~}c@{}}
\toprule
& \multicolumn{4}{c}{CUB-200-2011} & \multicolumn{4}{c}{aYahoo} & 
\multicolumn{4}{c}{AwA2} &
\multicolumn{4}{c}{aPascal} 
\\
\cmidrule(lr){2-5}
\cmidrule(lr){6-9}
\cmidrule(lr){10-13}
\cmidrule(lr){14-17}
Method & Bits & Dim & Acc & $d_D(F)$ & Bits & Dim & Acc & $d_D(F)$ 
& Bits & Dim & Acc & $d_D(F)$ & Bits & Dim & Acc & $d_D(F)$\\
\midrule
SDT~\citep{frosst2017distilling} & 4&512&9.80\% & 33.01& 3& 512& 17.58\% &31.90 &2 &256 & 7.21\%&66.17 & 2&128&30.48\%&31.53  \\
ANT~\citep{pmlr-v97-tanno19a} & 2& 512&11.13\% &33.69 &4 & 512&35.66\% &39.35& 3& 256&14.17\%&69.27 &3&256&31.70\%&39.24  \\
PMG~\cite{zhang2021progressive} & 4&512 & 86.41\%&39.06 & 3&512 & 96.14\% & 43.68 &1&256 &96.14\%&71.70&2&128&\textbf{79.22\%}&45.82  \\
WS-DAN~\citep{hu2019better} & 4& 256&87.54\% &37.90 &2&256&96.82\%&56.90
&4 & 512& 96.22\% & 78.08&3 &512&79.03\%&57.44  \\
DFL-CNN~\citep{wang2018learning} &3 & 512& 85.88\%& 37.56&3 & 512& 88.16\%& 30.06& 4& 256&89.20\% & 65.01&4&256&72.82\%& 43.13 \\
Nts~\citep{yang2018learning} &2 & 128 & 85.42\% &40.66 & 3& 128 & 98.11\% & 55.27&  4&256 &96.12\%&79.40&3&128&76.10\%&53.42  \\
API-Net~\citep{zhuang2020learning} &3&512 & \textbf{88.21\%} &36.77& 3 &  512& \textbf{98.72\%} & 57.26  & 4 & 256 & \textbf{96.25\%} &79.85 &  3&256&79.15\%&57.58  \\
ProtoPNet~\citep{Chen2019ThisLL} &3 & 512 & 78.82\% & 35.31 &4 & 256 & 85.62\% & 49.67 & 2 & 512 & 89.11\% & 76.42 &  2 &512&73.52\%&49.66  \\
Region Group \citep{huang2020interpretable} &3 &512 & 86.10\% &40.45& 2 &  512& 97.70\%& 58.90 & 3 & 512 & 96.11\% &80.27&3&256&77.27\%&56.80  \\
Quantized network & 4 & 512&80.08\%& 40.22& 1 &512& 98.08\%&55.75&2&512&95.30\%&80.32&2&256&74.80\%&56.30 \\
Quantized net + DT & 4&512  & 39.04\%& 32.11& 1 & 512 &58.24\% & 35.77& 2&512&52.09\%&68.77&2&256&47.65\%&33.82 \\ 
Quantized net + L1 & 4 & 512 & 80.61\% & 41.82& 1 & 512&97.57\% & 59.20&  2 &512  &94.81\%&83.31&2&256&75.06\%&59.80   \\
Our method & 4 & 512 & 79.82\% & \textbf{42.58} & 1 & 512 &96.86\% & \textbf{61.46} & 2 & 512 & 94.48\% &\textbf{84.67}&2&256&75.09\%&\textbf{61.72}\\
\bottomrule
\end{tabular}
\label{tab:main_results}}
\end{table}

\paragraph{Quantitative Analysis}\label{sec:exp1} 
{In Tab.~\ref{tab:main_results} we report the results of our method and the various baselines on each one of the datasets. As can be seen, our method has a significantly higher feature fidelity score, compared to the baselines across all datasets.} {We note that ProtoPNet~\citep{Chen2019ThisLL} makes use of cropping-based preprocessing. In order to fairly compare its results with the rest of the methods, we omitted using the preprocessing on all datasets except for CUB-200-2011. As can be seen, our method achieves a fidelity score that is higher than ProtoPNet's on CUB-200-2011, even though this kind of supervision helps ProtoPNet achieve the most competitive feature fidelity score among the baselines. On the other datasets, where no such crops were available, and ProtoPNet is trained on the entire frame, the feature fidelity score of ProtoPNet is significantly worse than that of our method.} 


\paragraph{Ablation Study} 
We conducted several ablation studies to validate the soundness of our method. 
Throughout the ablations, we compared our method with three of its variations: (i) a quantized network trained to minimize the cross-entropy loss, (ii) a quantized network trained to minimize the cross-entropy loss and the $L_1$ regularization loss { and (iii) a quantized network with a differentiable soft decision tree (SDT) on top of it, trained to minimize the cross-entropy loss. {In all three cases, the neural network is initialized with weights pre-trained on ILSVRC2012~\citep{10.1007/s11263-015-0816-y}.} We used the online learning decision trees of~\citet{10.1145/347090.347107}. 
}  The quantized network uses the same architecture as our model and is trained with the same hyperparameters. 
As can be seen in Tab.~\ref{tab:main_results}, our method significantly improves the recovery of the semantic attributes, at the small expense of a slight decrease in the classification accuracy, which is of second priority for our task.

To validate that the obtain performance gap in the fidelity score $d_{D}(F)$ is consistent across multiple configurations, we report in Fig~\ref{fig:barplots} the results of the baseline method (i) and our complete method. As can be seen, our method does not harm the models' accuracy, while it generally improves the feature fidelity score across the various configurations of feature dimensions and the number of bits used in the binarization process.

{A qualitative analysis studying which attributes emerge within the learned representations by applying our full method is presented in the supplementary material.}

\begin{figure*}[t]
\centering
\resizebox{1.\textwidth}{!}{
\begin{tabular}{@{}c@{}c@{}c@{}c@{}}
\includegraphics[width=10cm]{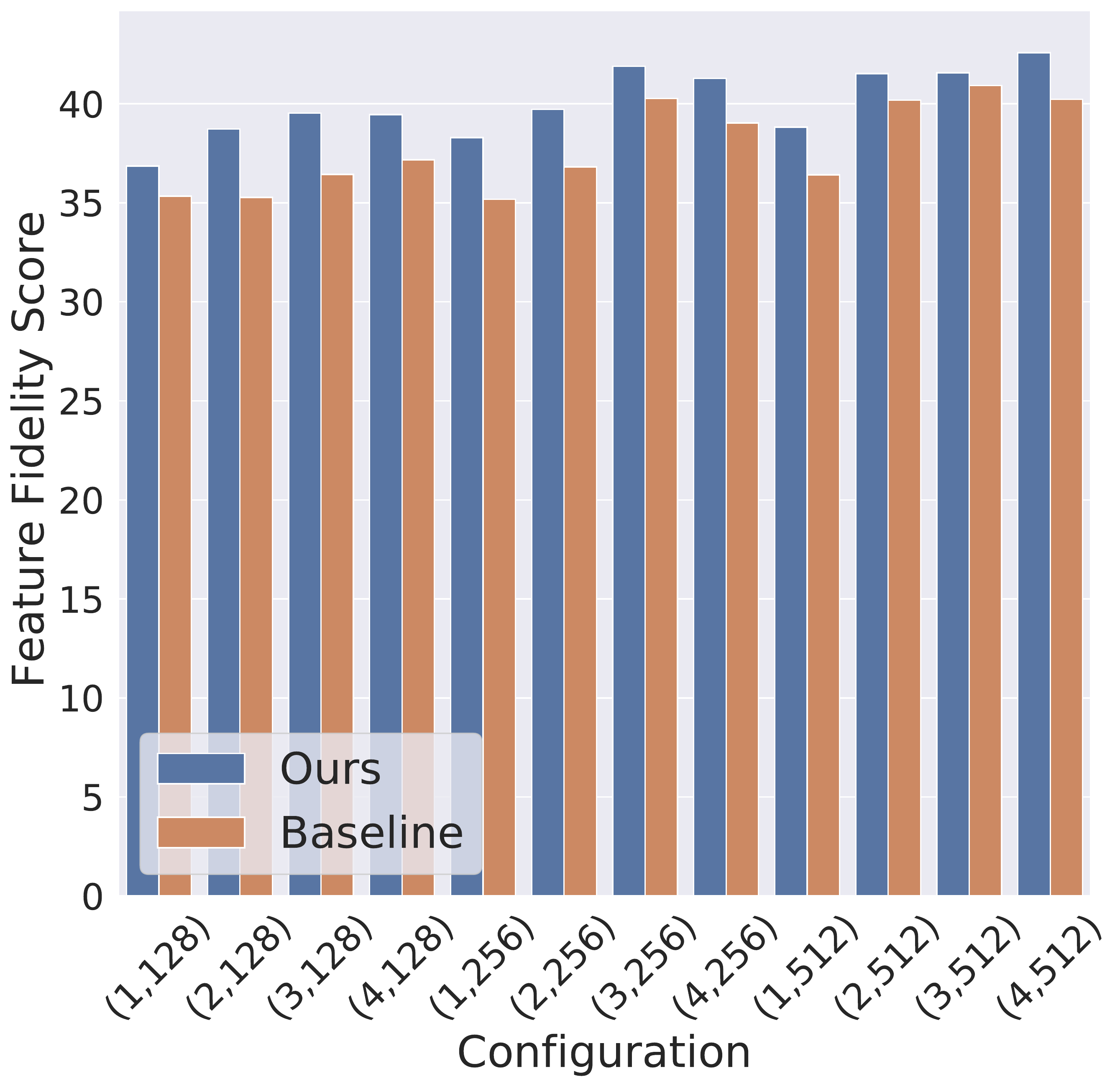} &
\includegraphics[width=10cm]{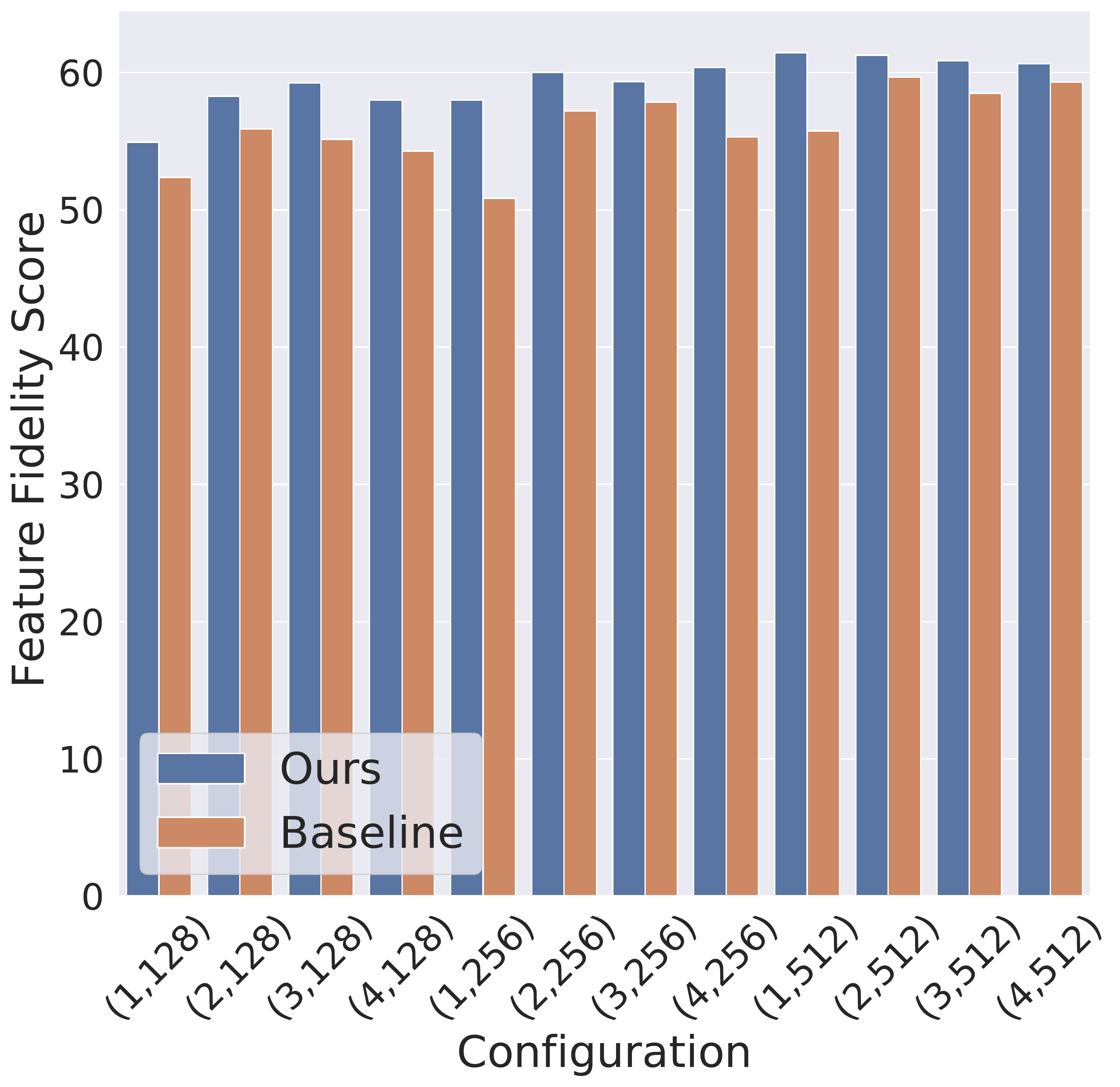} &
\includegraphics[width=10cm]{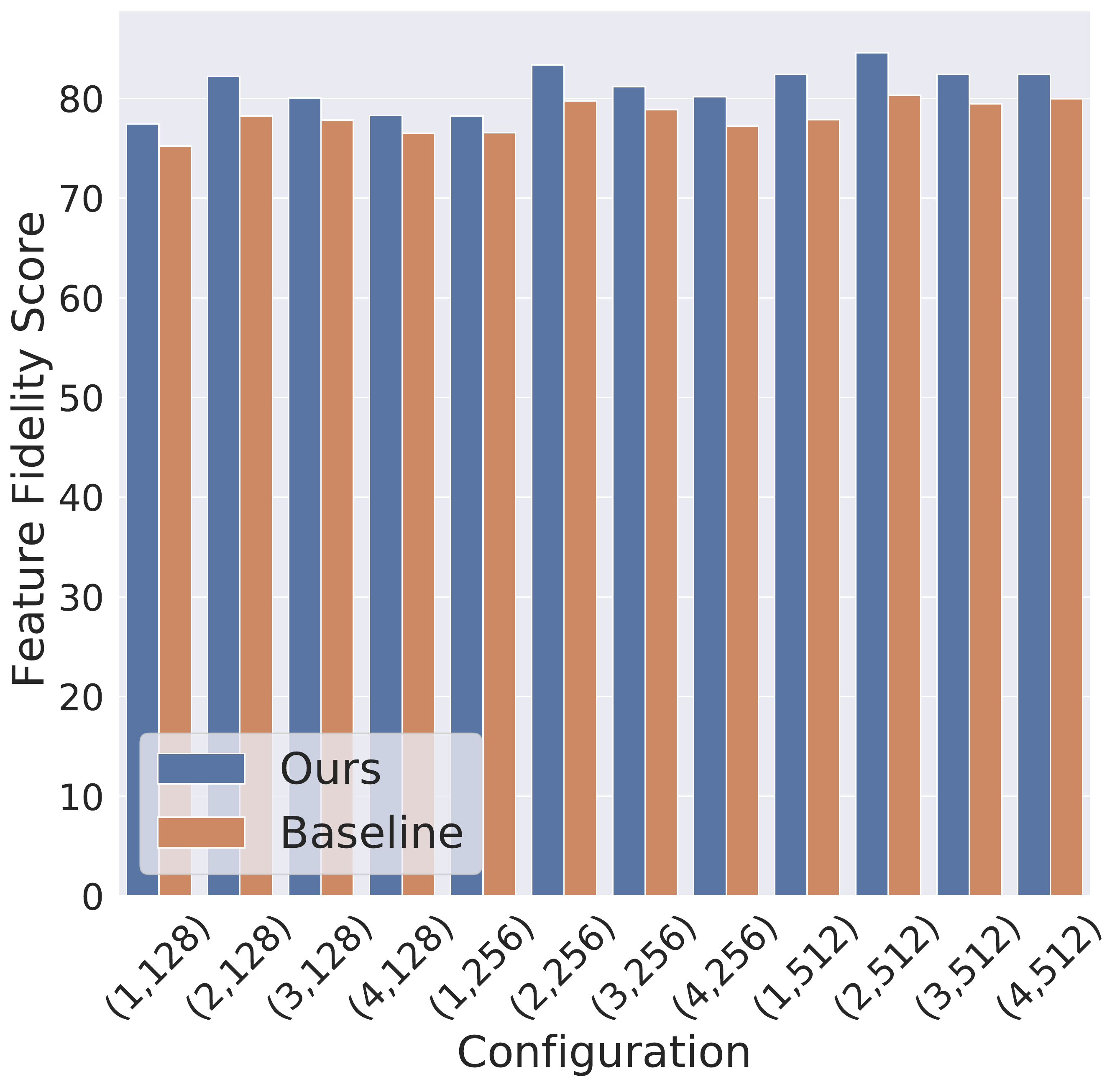} &
\includegraphics[width=10cm]{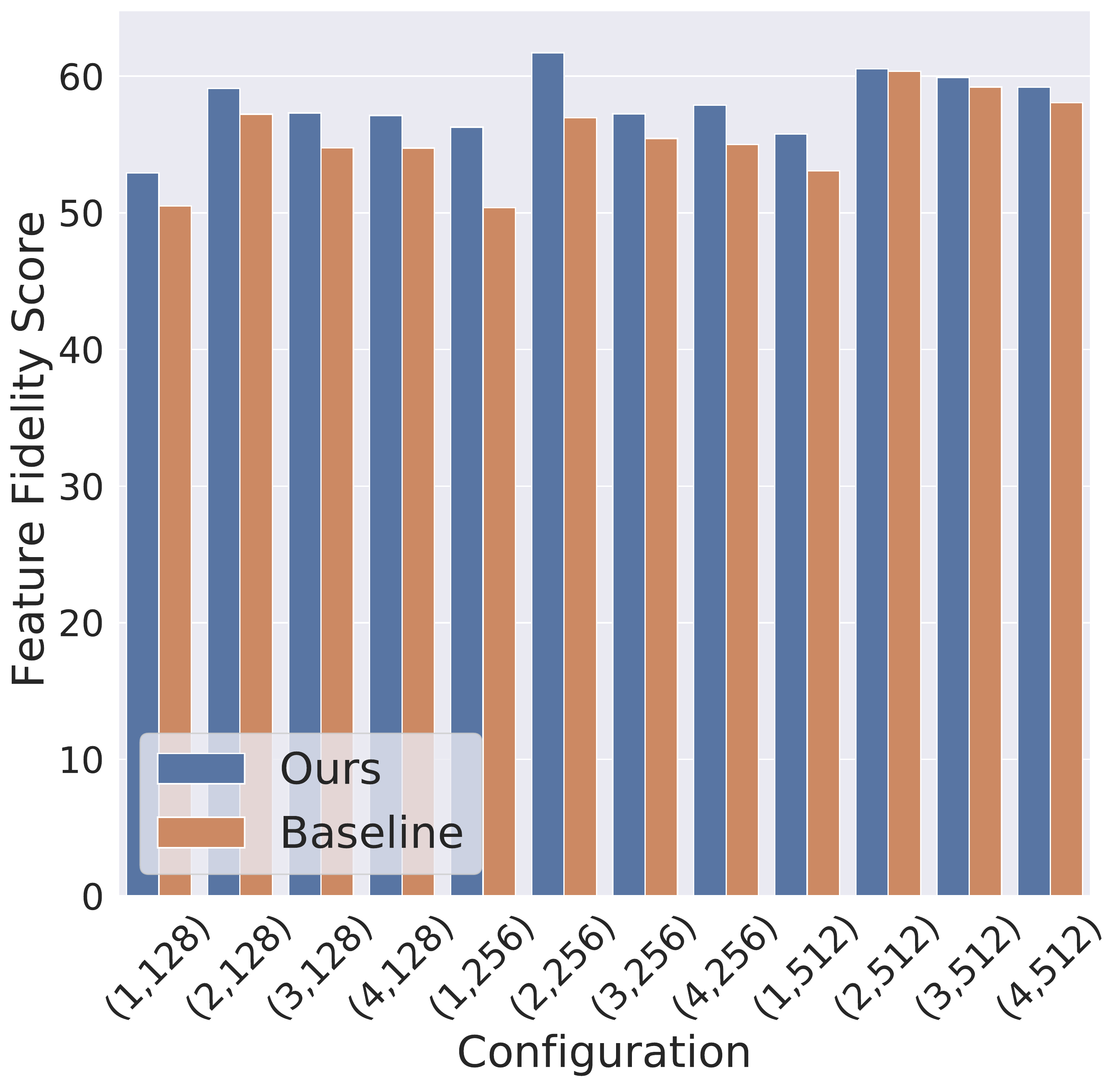} 
\\
\includegraphics[width=10cm]{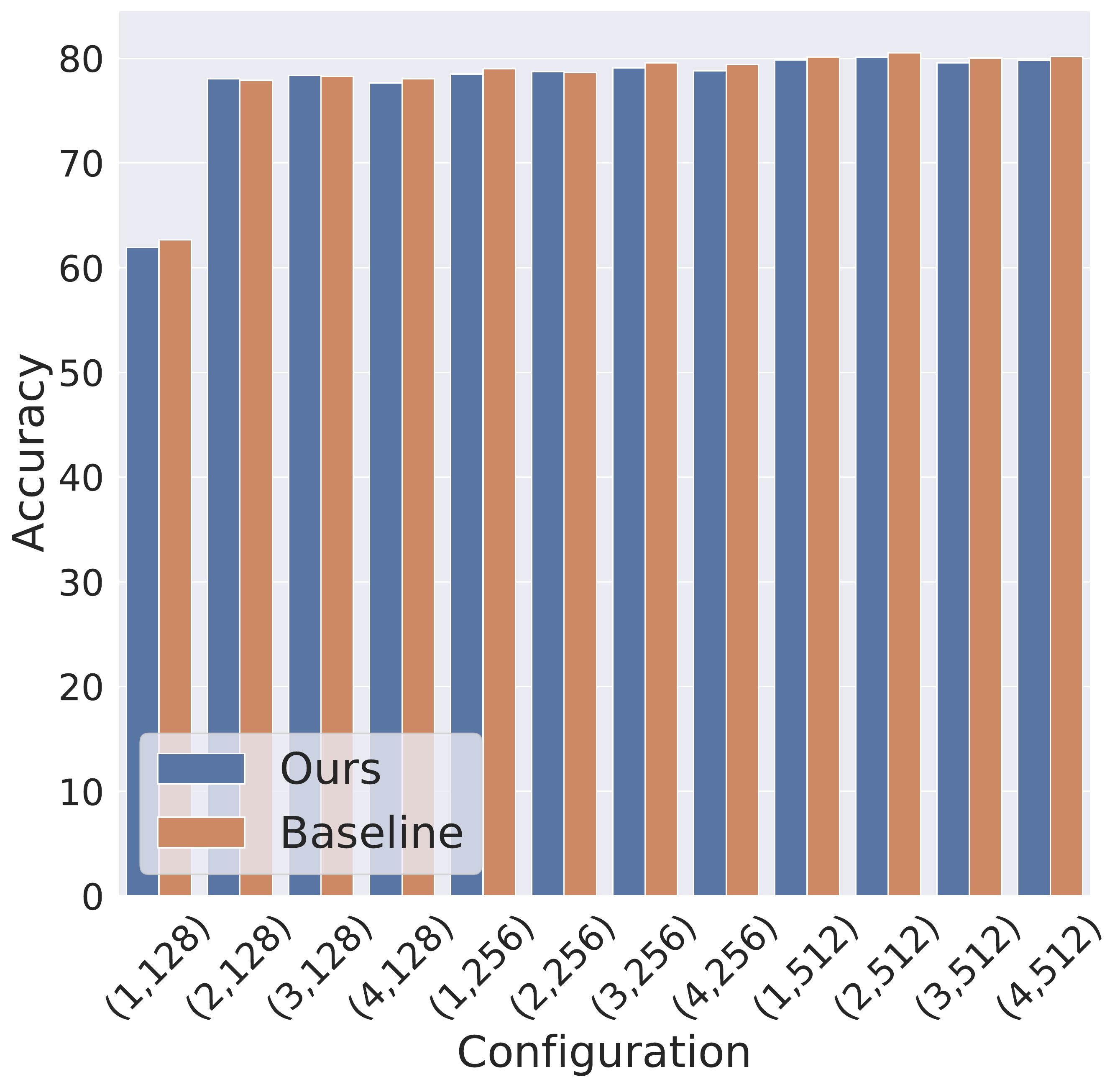} &
\includegraphics[width=10cm]{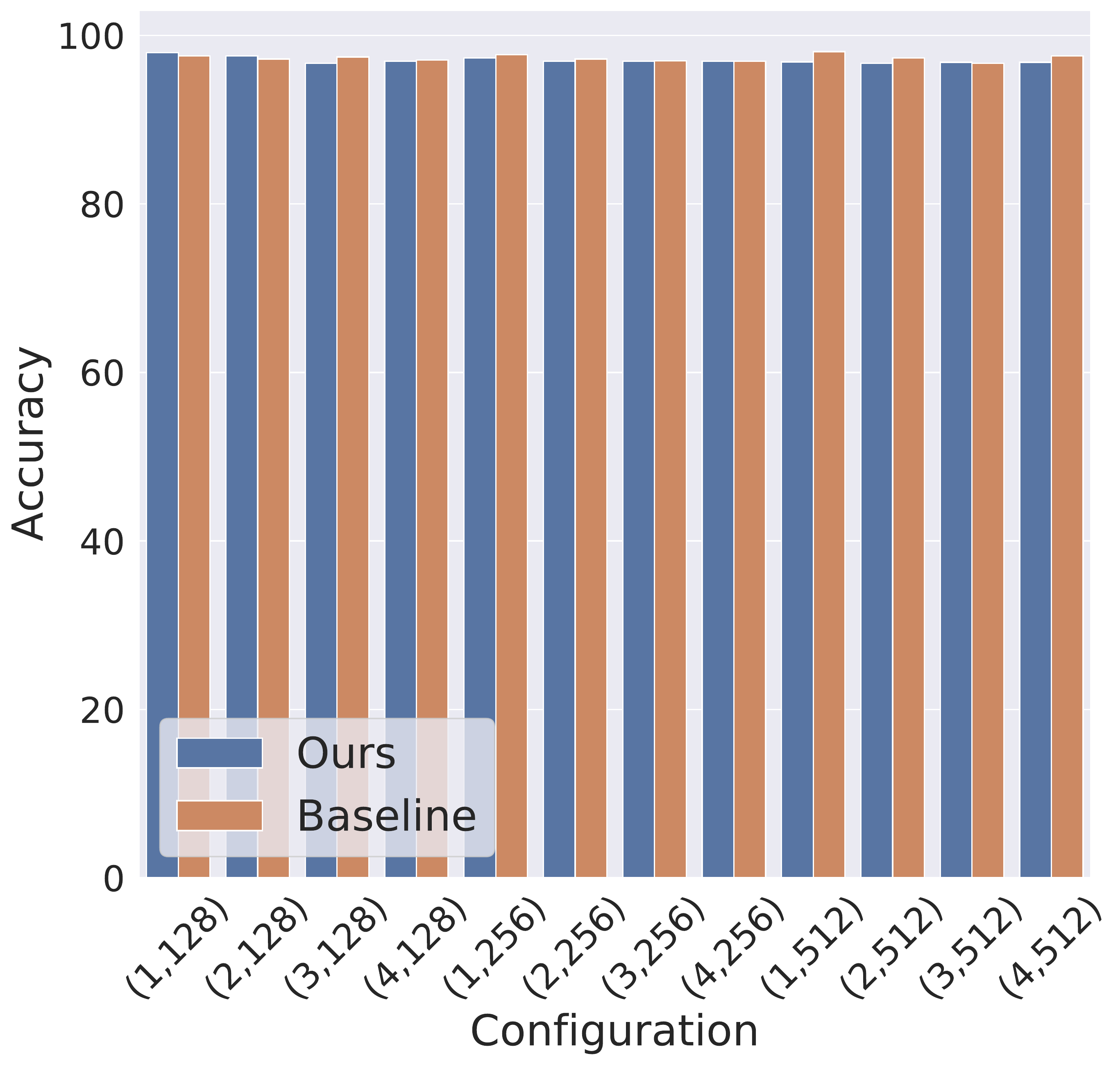} &
\includegraphics[width=10cm]{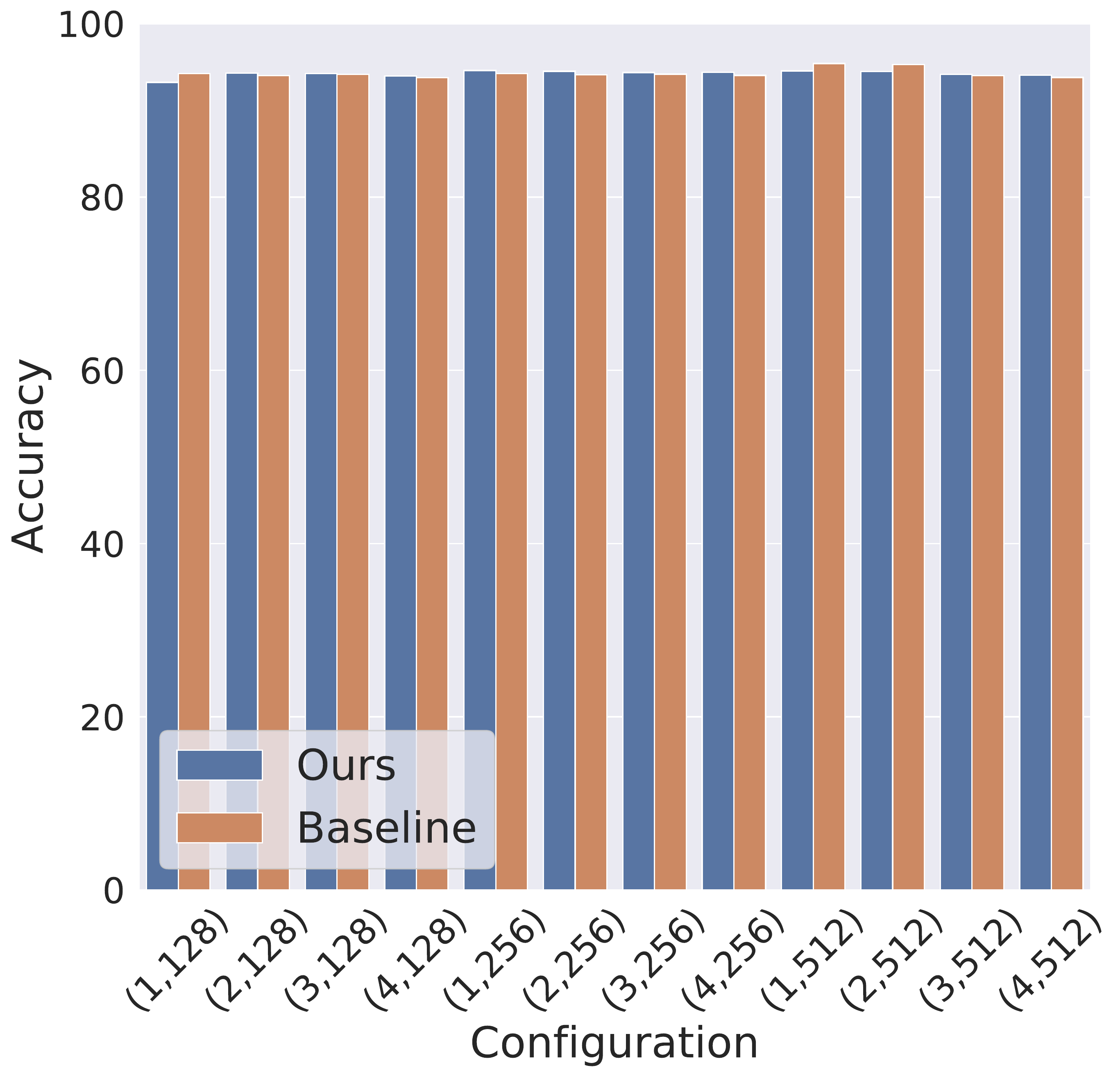} &
\includegraphics[width=10cm]{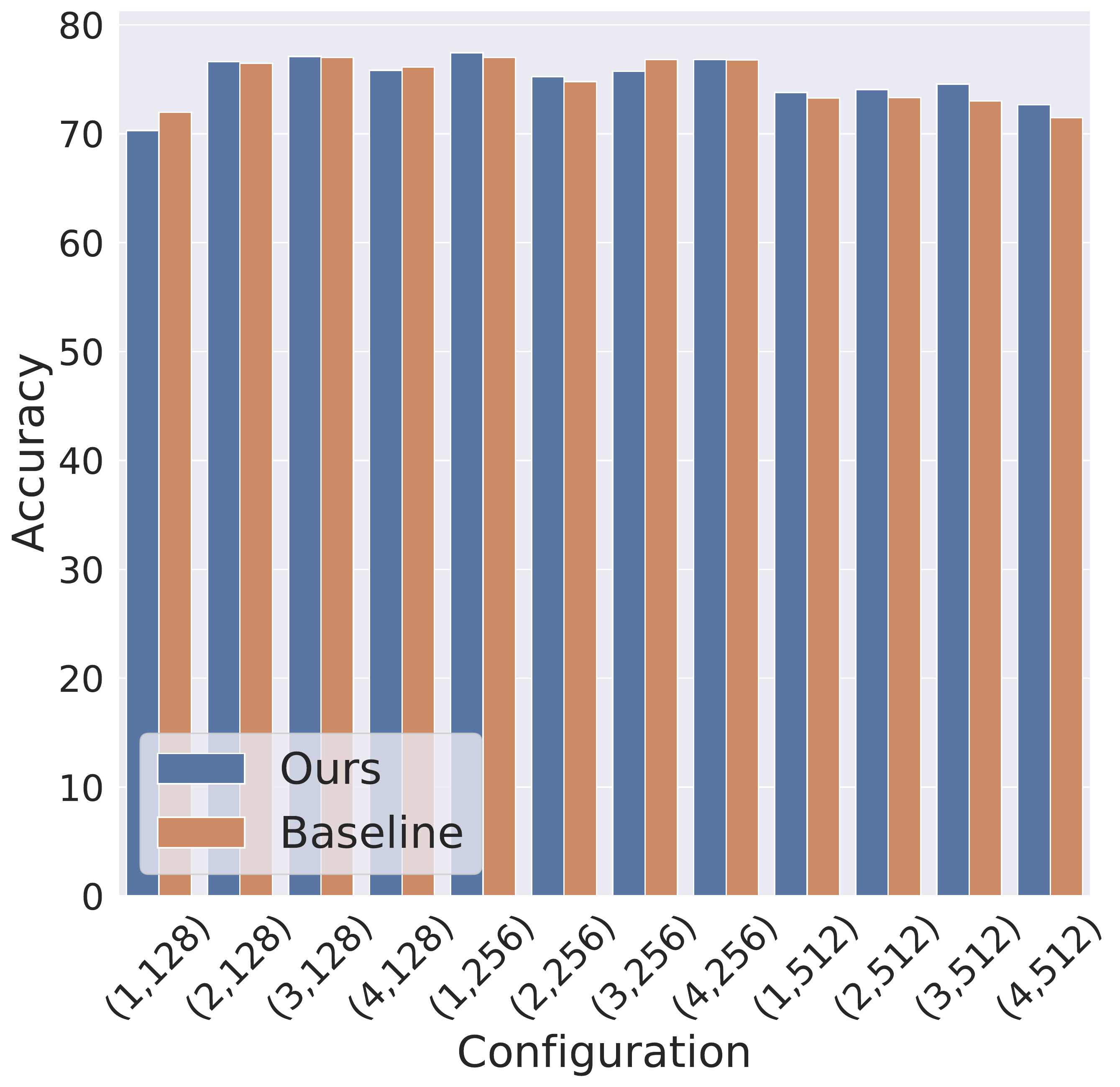} 
\\
(a) CUB-200-201 & (b) aYahoo & (c) AwA2 & (d) aPascal \\
\end{tabular}
}
\caption{{\bf Comparing the performance of our method (blue) and a standard quantized neural network (orange).} We compare the feature fidelity score ($d_{D}(F)$) {\bf(top)} and the classification accuracy rate {\bf(bottom)}. We plot the results for different configurations of feature dimensions $n$ and number of bits $k$ for binarization. Each configuration is labeled with $(n,k)$. }\label{fig:barplots}
\end{figure*}

\section{Limitations}

There are a few limitations to both our problem setting and algorithmic approach. First, the problem of `Weakly Supervised Recovery of Semantic Attributes' discussed in Sec.~\ref{sec:setup} is ill-posed in the general case. That is because we do not have an access to the ground-truth binary attributes, and potentially the function $f(x)$ could be replaced with a different function, even when assuming that $y(x)$ is a function of $f(x)$. For example, in general, if $y(x) = \textnormal{sign}(\sum_i f(x)_i)$, then, there are typically multiple functions $f(x)$ that could implement $y(x)$. In particular, it is typically unknown what the dimension of $F$ is ought to be. This is no different to other unsupervised learning settings, such as, clustering~\cite{Shalev-Shwartz:2014:UML:2621980,Ben-David_2018}, cross-domain mapping~\cite{CycleGAN2017,xia2016dual,pmlr-v70-kim17a,galanti2018the}, domain adaptation~\cite{DBLP:journals/ml/Ben-DavidBCKPV10,bendavid,DBLP:conf/alt/Mansour09,DBLP:conf/colt/MansourMR09}, causality and disentanglement~\citep{Peters2017,pmlr-v97-locatello19a} and sparse dictionary learning~\cite{7165675,8805108}, where ambiguity is an inherent aspect of the learning problem. 

This issue makes this problem challenging, especially when the number of ground-truth attributes is large as in the case of CUB-200-2011, or there are multiple redundant attributes in the dataset. 

To cope with this issue in unsupervised learning, we typically make assumptions on the structure of the target function, which narrows the space of functions captured by the algorithm. In our case, we assume that $y(x)$ can be represented as $t(f(x))$, where $t$ is a decision tree of a small depth.

\section{Conclusions}

In this paper, we introduced the problem of weakly supervised recovery of semantic attributes. This problem explores the emergence of interpretable features quantitatively, based on a set of semantic features that is unseen during training.  We present evidence that methods that are tailored to extract semantic attributes do not necessarily perform well in this metric and suggest a new approach to solving this problem. Our method is based on learning, concurrently, in the soft-intersection of two hypothesis classes: a neural network for obtaining its classification power, and a decision tree to learn features that lend themselves to classification by short logical expressions. We demonstrate that learning this way is more effective than other methods for the new task and that learning with the new regularization scheme improves the fidelity of the obtained feature map.

\section*{Acknowledgment}
This project has received funding from the European Research Council (ERC) under the European Unions Horizon 2020 research and innovation programme (grant ERC CoG 725974).

\bibliography{arxiv}
\bibliographystyle{plainnat}
\newpage
\appendix

\section{Uniform Quantization}

The uniform quantization method is fully described in Alg.~\ref{alg:quant}. For a given function $p$, we denote by $\partial p$ the estimated gradients of $p$ w.r.t $x$. In Alg.~\ref{alg:quant}, the estimated gradient of $\tilde{z}$ is computed as the gradient of $\min(\max(z_{init},q_{\min}),q_{\max})$ (since $\partial \mathrm{round}(x) = 1$ when applying STE~\cite{bengio2013estimating}), which is simply $\mathds{1}[z_{init} \in (q_{\min},q_{\max})] \cdot \frac{\partial z_{init}}{\partial x}$, where $\frac{\partial z_{init}}{\partial x}$ is the gradient of $z_{init}$ w.r.t $x$. $\partial \tilde{q}_i$ and $\partial{q}$ are defined similarly, see lines 10-11 in Alg.~\ref{alg:quant}.

\begin{algorithm}[t]
 \caption{The uniform quantization method (forward and backward passes)}
 \label{alg:quant}
 \begin{algorithmic}[1]
 \REQUIRE $x$ is tensor to be quantized; $r$ number of bits.
 \STATE $q_{\min}, q_{\max} = 0,  2^{r} - 1$;
 \STATE $x_{\min}, x_{\max} = \min_j \{x_j\}, \max_j \{x_j\}$;
 \STATE $s=\frac{x_{\max} - x_{\min}}{q_{\max} - q_{\min}}$;
 \STATE $z_{\text{init}} = \frac{q_{\min} - x_{\min}}{s}$;
 \STATE $z = \min(\max(z_{\text{init}},q_{\min}),q_{\max})$;
 \STATE $\tilde{z} = \mathrm{round}(z)$;
 \STATE $\forall i:~\tilde{q}_i = \tilde{z} + \frac{x_i}{s}$;
 \STATE $\forall i:~q_i = \mathrm{round}(\min(\max(\tilde{q}_i, q_{\min}), q_{\max}))$;
 \STATE $\partial \tilde{z} = \mathds{1}[z_{init} \in (q_{\min},q_{\max})] \cdot \frac{\partial z_{init}}{\partial x}$; 
 \STATE $\forall i:~\partial \tilde{q}_i = \partial \tilde{z} + \frac{\partial (x_i/s)}{\partial x}$; 
 \STATE $\forall i:~\partial q_i = \mathds{1}[\tilde{q_i} \in (q_{\min},q_{\max})] \cdot \partial \tilde{q}_i$;
 \STATE {\bf return} $q$, $\partial q$;

  \end{algorithmic}
\end{algorithm}


\section{Additional Experimental Details}

{\bf Datasets\quad} (i) The aYahoo dataset~\citep{farhadi2009describing} consists of 1850 training images and 794 test images, (ii) the aPascal dataset~\citep{farhadi2009describing} has 2869 training images and 2227 test images, (iii) in the AwA2 dataset~\citep{xian2018zero} we used 26125 training images and 11197 test images, and (iv) the CUB-200-2011 dataset~\citep{WahCUB_200_2011} contains 5994 training images and 5794 test images. We used the standard train/test splits for each one of the datasets, except for AwA2, where we follow a 0.7/0.3 random split because there is no standard partition for this dataset.

{\bf Runtime and Infrastructure\quad} The experiments were run on three GeForce RTX 3090 GPUs. Our method completes an epoch every 7 minutes on aYahoo, 10 minutes on aPascal, 1 hour on AwA2, and 1 hour on CUB-200-2011. On average our method runs for approximately 12 epochs on both aYahoo and aPascal, 2-4 epochs on AwA2, and 35-40 on CUB-200-2011. 

\section{Qualitative Analysis}

We provide a qualitative analysis of the advantage of our method over a simple neural network with a quantized representation layer. In this experiment, for each ground-truth attribute $q_i(x)$ in the dataset, we compute the degree it appears in the learned binarized representation layer $b \circ q \circ F$ for each method: $d_{D}(q_i;b \circ q \circ F)$. In Figs.~\ref{fig:ayahoo}-\ref{fig:awa} we plot the values of $d_{D}(q_i;b \circ q \circ F)$ for each attribute in the dataset as a heatmap. We also report the number of attributes for which our method achieves a higher score (by at least $\epsilon=0.05$) and vice versa. 

\begin{figure}[!ht]
\centering
\includegraphics[width=1\textwidth]{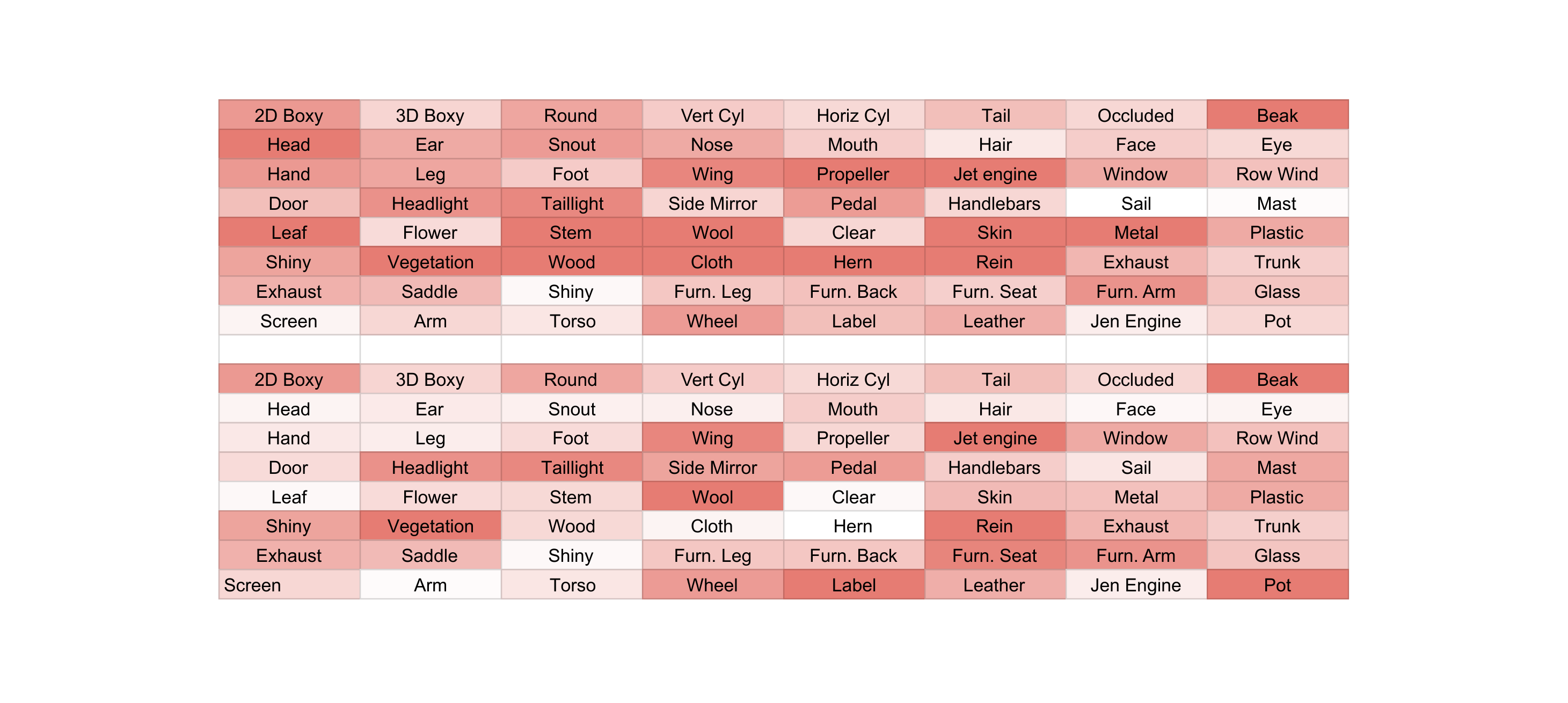}
\caption{{\bf Heatmaps of $d_{D}(q_i;b \circ q \circ F)$ on the aYahoo dataset~\citep{farhadi2009describing}.} In {\bf (top)} we report the results of our method and {\bf (bottom)} reports the results of a standard neural network with a quantized representation layer. Our method has a higher score on 23 attributes and a lower score on 4 attributes.}
\label{fig:ayahoo}
\end{figure}
\begin{figure}[!ht]
\centering
\includegraphics[width=1\textwidth]{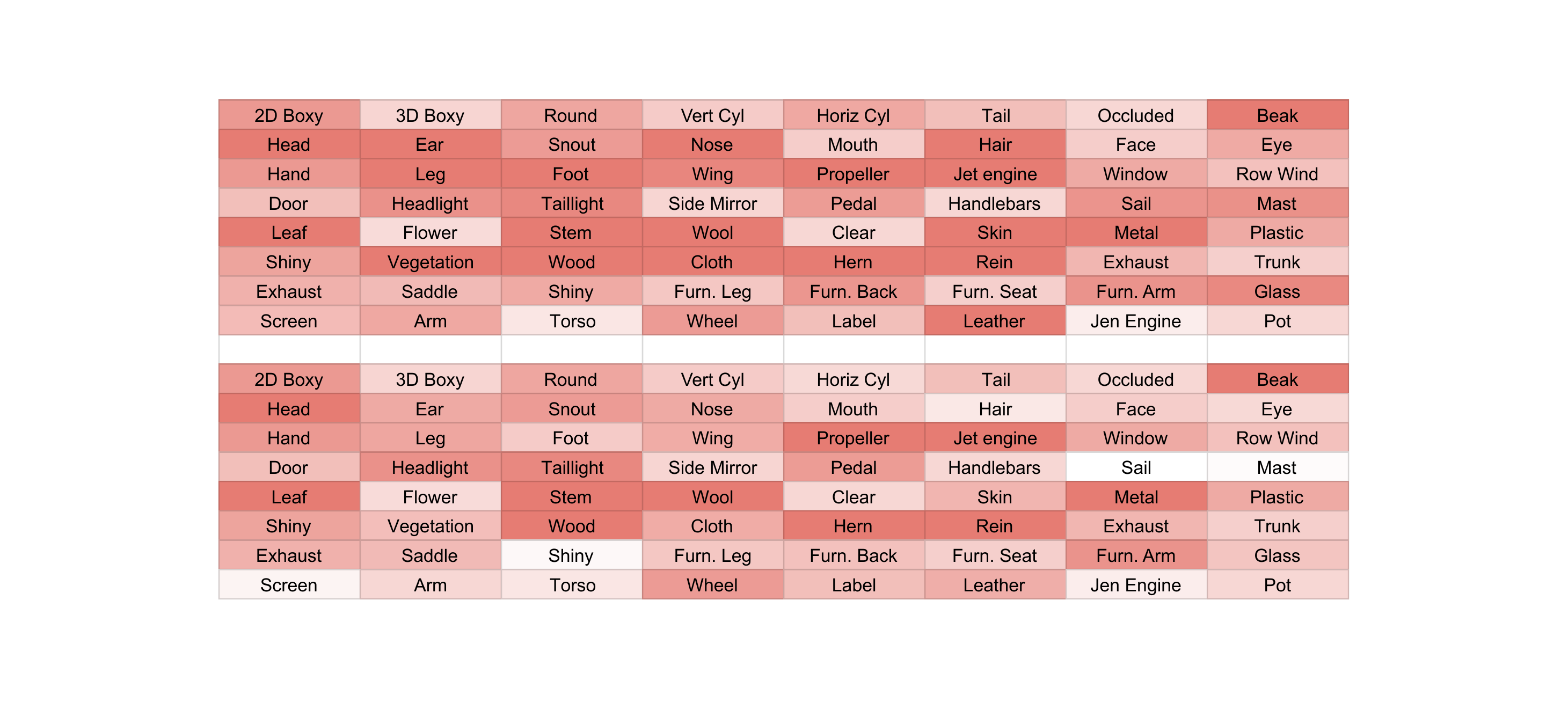}
\caption{{\bf Heatmaps of $d_{D}(q_i;b \circ q \circ F)$ on the aPascal dataset~\citep{farhadi2009describing}.} In {\bf (top)} we report the results of our method and {\bf (bottom)} reports the results of a standard neural network with a quantized representation layer. Our method has a higher score on 38 attributes and a lower score on 3 attributes. }
\label{fig:apascal}
\end{figure}
\begin{figure}[!ht]
\centering
\includegraphics[width=1\textwidth]{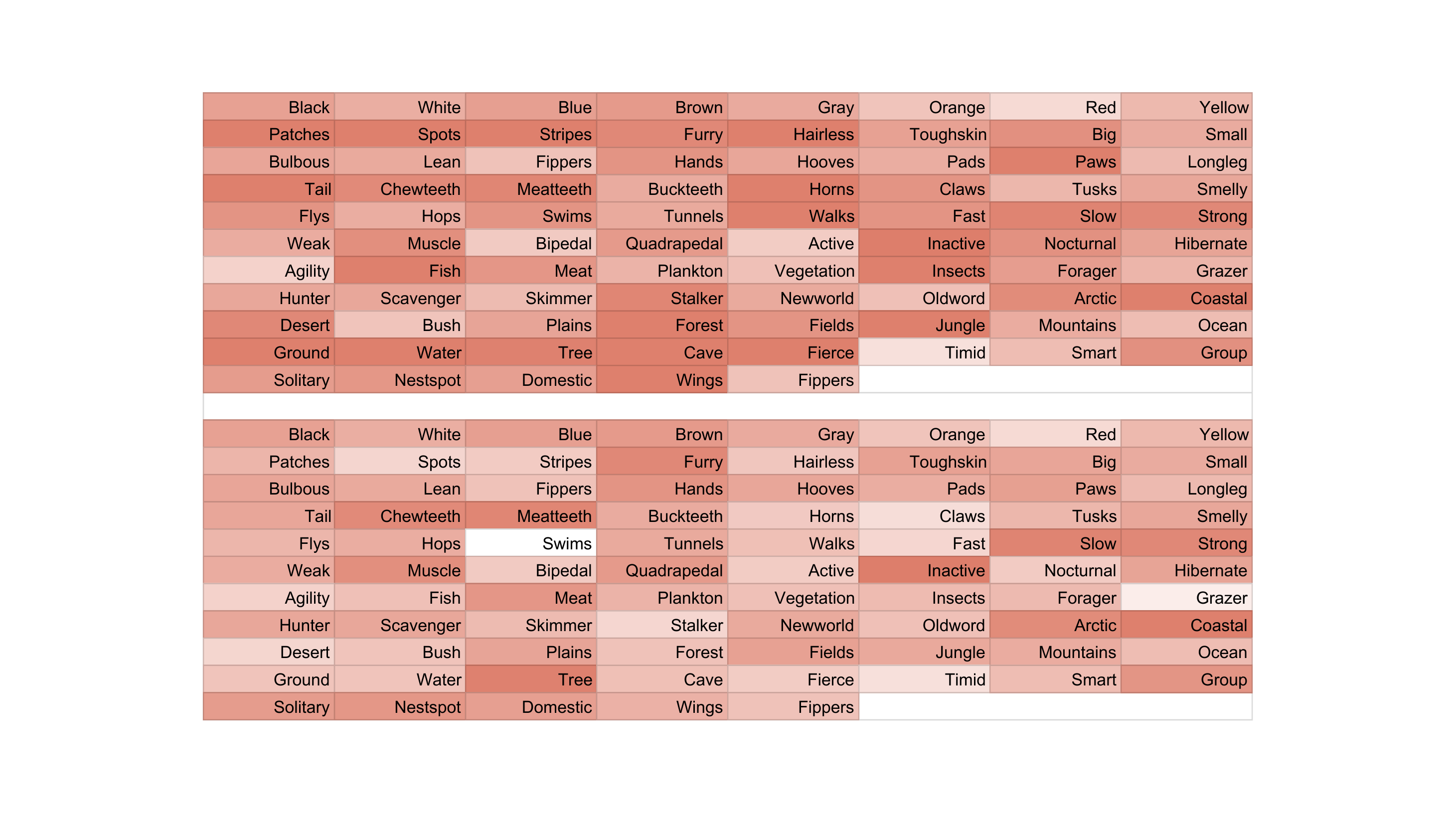}
\caption{{\bf Heatmaps of $d_{D}(q_i;b \circ q \circ F)$ on the AwA2 dataset~\citep{xian2018zero}.} In {\bf (top)} we report the results of our method and {\bf (bottom)} reports the results of a standard neural network with a quantized representation layer. Our method has a higher score on 48 attributes and a lower score on 3 attributes.}
\label{fig:awa}
\end{figure}



\newpage

\section{Proofs of Props.~\ref{lem:1} and~\ref{lem:2}}

\lemOne*

\begin{proof}
First, we would like to prove that $\mathcal{Q}(\theta_{t},\omega_{t})$ is monotonically decreasing. We notice that for each $t \in \mathbb{N}$, we have:
\begin{equation}
\mathcal{Q}(\theta_{t},\omega_{t}) \leq \mathcal{Q}(\theta_{t},\omega_{t-1})
\end{equation}
since $\omega_{t}$ is the global minimizer of $\mathcal{Q}(\theta_{t},\omega)$. In addition, by the proof of (cf.~\citet{10.5555/2670022}, Thm.~2.1.14), we have:
\begin{equation}\label{eq:a}
\begin{aligned}
\mathcal{Q}(\theta_{t},\omega_{t-1}) 
\leq \mathcal{Q}(\theta_{t-1},\omega_{t-1}) - \eta \| \nabla_{\theta}\mathcal{Q}(\theta_{t},\omega_{t-1})\|^2_2
\leq \mathcal{Q}(\theta_{t-1},\omega_{t-1}),
\end{aligned}
\end{equation}
where $\eta := \mu (1-0.5\beta \mu) > 0$. Therefore, we conclude that  $\mathcal{Q}(\theta_{t},\omega_{t})$ is indeed monotonically decreasing. Since $\mathcal{Q}$ is non-negative, we conclude that $\mathcal{Q}(\theta_t,\omega_t)$ is a convergent sequence. In particular, $\theta_{t}$ is bounded. Otherwise, $\mathcal{Q}(\theta_{t},\omega_{t}) \geq \mathcal{L}_{\mathcal{S}}[h_{\theta_{t}},y] \to \infty$ in contradiction to the fact that $\mathcal{Q}(\theta_{t},\omega_{t})$ is a convergent sequence. 

Next, by applying the convexity of $\mathcal{Q}(\theta,\omega_{t-1})$ (as a function of $\theta$), we have:
\begin{equation}\label{eq:b}
\begin{aligned}
\mathcal{Q}(\theta_{t},\omega_{t-1}) - \mathcal{Q}(\theta^*_{t},\omega_{t-1}) \leq \langle \nabla_{\theta} \mathcal{Q}(\theta,\omega_{t-1}) , \theta_{t} - \theta^*_{t} \rangle 
\leq \|\nabla_{\theta}\mathcal{Q}(\theta_{t},\omega_{t-1}) \|_2 \cdot \|\theta_{t} - \theta^*_{t} \|_2, 
\end{aligned}
\end{equation}
where $\theta^*_{t} = \arg\min_{\theta} \mathcal{Q}(\theta,\omega_{t})$. By combining Eqs.~\eqref{eq:a} and~\eqref{eq:b}, we have:
\begin{equation}
\begin{aligned}
&\mathcal{Q}(\theta_{t},\omega_{t}) 
\leq \mathcal{Q}(\theta_{t},\omega_{t-1}) 
\leq \mathcal{Q}(\theta_{t-1},\omega_{t-1}) - \frac{\eta}{\|\theta_{t} - \theta^*_{t}\|^2_2} \left( \mathcal{Q}(\theta_{t-1},\omega_{t-1}) - \min_{\theta} \mathcal{Q}(\theta,\omega_{t-1}) \right)^2
\end{aligned}
\end{equation}
In particular, 
\begin{equation}
\begin{aligned}
\mathcal{Q}(\theta_{t-1},\omega_{t-1}) - \mathcal{Q}(\theta_{t},\omega_{t}) \geq \frac{\eta}{\|\theta_{t} - \theta^*_{t}\|^2_2} \left( \mathcal{Q}(\theta_{t-1},\omega_{t-1}) - \min_{\theta} \mathcal{Q}(\theta,\omega_{t-1}) \right)^2
\end{aligned}
\end{equation}
Since the left hand side tends to zero and the right hand side is lower bounded by zero, by the sandwich theorem, the right hand side tends to zero as well. Since both $\theta_{t}$ and $\theta^*_{t}$ are bounded sequences ($\{\arg\min_{\theta} \mathcal{Q}(\theta,\omega) \mid \omega \in \Omega\}$ is well-defined and bounded), we conclude that $\lim_{t \to \infty} \mathcal{Q}(\theta_{t},\omega_{t}) = \lim_{t \to \infty} \min_{\theta} \mathcal{Q}(\theta,\omega_{t})$. We also have: $\mathcal{Q}(\theta_{t},\omega_{t}) = \min_{\omega} \mathcal{Q}(\theta_{t},\omega)$ by definition, and therefore, $\lim_{t \to \infty} \mathcal{Q}(\theta_{t},\omega_{t}) = \lim_{t \to \infty} \min_{\omega} \mathcal{Q}(\theta_{t},\omega)$ as well.
\end{proof}

\lemTwo*

\begin{proof}
Firstly, since $\mathcal{Q}(\theta,\omega)$ is a twice-continuously differentiable, Lipschitz continuous and $\beta$-smooth function, by Prop.~3.4 and Cor.~3.1 in~\citep{song2017block}, $(\theta_{t},\omega_{t})$ converge to a local minima $(\hat{\theta},\hat{\omega})$ of $\mathcal{Q}$. Therefore, it is left to show that $(\hat{\theta},\hat{\omega})$ is also an equilibrium point. We note that $\mathcal{Q}(\theta,\omega)$ is element-wise convex and $\beta$-smooth. By the proof of (cf.~\citet{10.5555/2670022}, Thm.~2.1.14), we have:
\begin{equation}\label{eq:c}
\begin{aligned}
\mathcal{Q}(\theta_{t+1},\omega_{t}) \leq& \mathcal{Q}(\theta_{t},\omega_{t}) - \eta \| \nabla_{\theta}\mathcal{Q}(\theta_{t},\omega_{t})\|^2_2,
\end{aligned}
\end{equation}
where $\eta := \mu (1-0.5\beta \mu) > 0$. By applying the convexity of $\mathcal{Q}(\theta,\omega_{t})$ (as a function of $\theta$), we have:
\begin{equation}\label{eq:d}
\begin{aligned}
\mathcal{Q}(\theta_{t},\omega_{t}) - \mathcal{Q}(\theta^*_{t},\omega_{t}) \leq& \langle \nabla_{\theta} \mathcal{Q}(\theta,\omega_{t}) , \theta_{t} - \theta^*_{t} \rangle \\
\leq& \|\nabla_{\theta}\mathcal{Q}(\theta_{t},\omega_{t}) \|_2 \cdot \|\theta_{t} - \theta^*_{t} \|_2, \\
\end{aligned}
\end{equation}
where $\theta^*_{t} = \arg\min_{\theta} \mathcal{Q}(\theta,\omega_{t})$ and $\omega^*_{t} = \arg\min_{\omega} \mathcal{Q}(\theta_{t+1},\omega)$. By combining Eqs.~\eqref{eq:c} and~\eqref{eq:d}, we have:
\begin{equation}
\begin{aligned}
\mathcal{Q}(\theta_{t+1},\omega_{t}) \leq \mathcal{Q}(\theta_{t},\omega_{t}) - \frac{\eta}{\|\theta_{t} - \theta^*_{t}\|^2_2} \left( \mathcal{Q}(\theta_{t},\omega_{t}) - \min_{\theta} \mathcal{Q}(\theta,\omega_{t}) \right)^2
\end{aligned}
\end{equation}
and similarly, we have:
\begin{equation}
\begin{aligned}
\mathcal{Q}(\theta_{t+1},\omega_{t+1}) \leq \mathcal{Q}(\theta_{t+1},\omega_{t})  - \frac{\eta}{\|\omega_{t} - \omega^*_{t}\|^2_2} \left( \mathcal{Q}(\theta_{t+1},\omega_{t}) - \min_{\omega} \mathcal{Q}(\theta_{t+1},\omega) \right)^2
\end{aligned}
\end{equation}

In particular, we have:
\begin{equation}\label{eq:beta}
\begin{aligned}
\mathcal{Q}(\theta_{t+1},\omega_{t+1}) \leq& \mathcal{Q}(\theta_{t},\omega_{t}) - \frac{\eta}{\|\theta_{t} - \theta^*_{t}\|^2_2} \left( \mathcal{Q}(\theta_{t},\omega_{t}) - \min_{\theta} \mathcal{Q}(\theta,\omega_{t}) \right)^2 \\
&- \frac{\eta}{\|\omega_{t} - \omega^*_{t}\|^2_2} \left( \mathcal{Q}(\theta_{t+1},\omega_{t}) - \min_{\omega} \mathcal{Q}(\theta_{t+1},\omega) \right)^2
\end{aligned}
\end{equation}
Therefore, the sequence $\mathcal{Q}(\theta_{t},\omega_{t})$ is monotonically decreasing. Since $\mathcal{Q}(\theta_{t},\omega_{t})$ is non-negative, it converges to some non-negative constant. By Eq.~\eqref{eq:beta}, we have:
\begin{equation}
\begin{aligned}
&\mathcal{Q}(\theta_{t+1},\omega_{t+1}) - \mathcal{Q}(\theta_{t},\omega_{t})\\
\geq& \frac{\eta}{\|\theta_{t} - \theta^*_{t}\|^2_2} \left( \mathcal{Q}(\theta_{t},\omega_{t}) - \min_{\theta} \mathcal{Q}(\theta,\omega_{t}) \right)^2 + \frac{\eta}{\|\omega_{t} - \omega^*_{t}\|^2_2} \left( \mathcal{Q}(\theta_{t+1},\omega_{t}) - \min_{\omega} \mathcal{Q}(\theta_{t+1},\omega) \right)^2
\end{aligned}
\end{equation}
Since the left hand side tends to zero and the right hand side is lower bounded by zero, by the sandwich theorem, the sequences $\frac{\eta}{ \|\theta_{t} - \theta^*_{t}\|^2_2} \left( \mathcal{Q}(\theta_{t},\omega_{t}) - \min_{\theta} \mathcal{Q}(\theta,\omega_{t}) \right)^2 $ and $ \frac{\eta}{\|\omega_{t} - \omega^*_{t}\|^2_2} \left( \mathcal{Q}(\theta_{t+1},\omega_{t}) - \min_{\omega} \mathcal{Q}(\theta_{t+1},\omega) \right)^2$ tend to zero. We note that $\theta_{t}$ and $\omega_{t}$ are convergent sequences, and therefore, are bounded as well. In addition, we recall that the sets $\{\arg\min_{\theta} \mathcal{Q}(\theta,\omega) \mid \omega \in \Omega\}$ and $\{\arg\min_{\omega} \mathcal{Q}(\theta,\omega) \mid \theta \in \Theta\}$ are well-defined and bounded. Hence, the terms $\|\theta_{t} - \theta^*_{t}\|^2_2$ and $\|\omega_{t} - \omega^*_{t}\|^2_2$ are bounded. Thus, we conclude that the sequences $ \left( \mathcal{Q}(\theta_{t},\omega_{t}) - \min_{\theta} \mathcal{Q}(\theta,\omega_{t}) \right)^2 $ and $\left( \mathcal{Q}(\theta_{t+1},\omega_{t}) - \min_{\omega} \mathcal{Q}(\theta_{t+1},\omega) \right)^2$ tend to zero. In particular,
\begin{equation}
\begin{aligned}
\lim_{t \to \infty} \mathcal{Q}(\theta_{t},\omega_{t}) 
&= \lim_{t \to \infty} \min_{\theta} \mathcal{Q}(\theta,\omega_{t}) \\
\lim_{t \to \infty} \mathcal{Q}(\theta_{t+1},\omega_{t})  &= \lim_{t \to \infty} \min_{\omega} \mathcal{Q}(\theta_{t},\omega)
\end{aligned}
\end{equation}
Since $\mathcal{Q}$ is a Lipschitz continuous function, we have: (i) $\lim_{t \to \infty} \mathcal{Q}(\theta_{t+1},\omega_{t}) = \lim_{t \to \infty} \mathcal{Q}(\theta_{t},\omega_{t}) = \mathcal{Q}(\hat{\theta},\hat{\omega})$, (ii) $\lim_{t \to \infty} \min_{\theta} \mathcal{Q}(\theta,\omega_{t}) =  \min_{\theta} \mathcal{Q}(\theta,\hat{\omega})$ and (iii) $\lim_{t \to \infty} \min_{\omega} \mathcal{Q}(\theta_{t},\omega) =  \min_{\omega} \mathcal{Q}(\hat{\theta},\omega)$. Therefore, we finally conclude that $(\hat{\theta},\hat{\omega})$ is an equilibrium point of $\mathcal{Q}$.
\end{proof}




\end{document}


\twocolumn[
\icmltitle{Supplementary Material \\ Intersection Regularization for Extracting Semantic Attributes}



\icmlsetsymbol{equal}{*}

\begin{icmlauthorlist}
\icmlauthor{Aeiau Zzzz}{equal,to}
\icmlauthor{Bauiu C.~Yyyy}{equal,to,goo}
\icmlauthor{Cieua Vvvvv}{goo}
\icmlauthor{Iaesut Saoeu}{ed}
\icmlauthor{Fiuea Rrrr}{to}
\icmlauthor{Tateu H.~Yasehe}{ed,to,goo}
\icmlauthor{Aaoeu Iasoh}{goo}
\icmlauthor{Buiui Eueu}{ed}
\icmlauthor{Aeuia Zzzz}{ed}
\icmlauthor{Bieea C.~Yyyy}{to,goo}
\icmlauthor{Teoau Xxxx}{ed}
\icmlauthor{Eee Pppp}{ed}
\end{icmlauthorlist}

\icmlaffiliation{to}{Department of Computation, University of Torontoland, Torontoland, Canada}
\icmlaffiliation{goo}{Googol ShallowMind, New London, Michigan, USA}
\icmlaffiliation{ed}{School of Computation, University of Edenborrow, Edenborrow, United Kingdom}

\icmlcorrespondingauthor{Cieua Vvvvv}{c.vvvvv@googol.com}
\icmlcorrespondingauthor{Eee Pppp}{ep@eden.co.uk}

\icmlkeywords{Machine Learning, ICML}

\vskip 0.3in
]



\printAffiliationsAndNotice{\icmlEqualContribution} 

\section{Uniform Quantization}

The uniform quantization method is fully described in Alg.~\ref{alg:quant}. For a given function $p$, we denote by $\partial p$ the estimated gradients of $p$ w.r.t $x$. In Alg.~\ref{alg:quant}, the estimated gradient of $\tilde{z}$ is computed as the gradient of $\min(\max(z_{init},q_{\min}),q_{\max})$ (since $\partial \mathrm{round}(x) = 1$ when applying STE~\cite{bengio2013estimating}), which is simply $\mathds{1}[z_{init} \in (q_{\min},q_{\max})] \cdot \frac{\partial z_{init}}{\partial x}$, where $\frac{\partial z_{init}}{\partial x}$ is the gradient of $z_{init}$ w.r.t $x$. $\partial \tilde{q}_i$ and $\partial{q}$ are defined similarly, see lines 10-11 in Alg.~\ref{alg:quant}.

\begin{algorithm}[t]
 \caption{The uniform quantization method (forward and backward passes)}
 \label{alg:quant}
 \begin{algorithmic}[1]
 \REQUIRE $x$ is tensor to be quantized; $r$ number of bits.
 \STATE $q_{\min}, q_{\max} = 0,  2^{r} - 1$;
 \STATE $x_{\min}, x_{\max} = \min_j \{x_j\}, \max_j \{x_j\}$;
 \STATE $s=\frac{x_{\max} - x_{\min}}{q_{\max} - q_{\min}}$;
 \STATE $z_{\text{init}} = \frac{q_{\min} - x_{\min}}{s}$;
 \STATE $z = \min(\max(z_{\text{init}},q_{\min}),q_{\max})$;
 \STATE $\tilde{z} = \mathrm{round}(z)$;
 \STATE $\forall i:~\tilde{q}_i = \tilde{z} + \frac{x_i}{s}$;
 \STATE $\forall i:~q_i = \mathrm{round}(\min(\max(\tilde{q}_i, q_{\min}), q_{\max}))$;
 \STATE $\partial \tilde{z} = \mathds{1}[z_{init} \in (q_{\min},q_{\max})] \cdot \frac{\partial z_{init}}{\partial x}$; 
 \STATE $\forall i:~\partial \tilde{q}_i = \partial \tilde{z} + \frac{\partial (x_i/s)}{\partial x}$; 
 \STATE $\forall i:~\partial q_i = \mathds{1}[\tilde{q_i} \in (q_{\min},q_{\max})] \cdot \partial \tilde{q}_i$;
 \STATE {\bf return} $q$, $\partial q$;

  \end{algorithmic}
\end{algorithm}


\section{Additional Experimental Details}

{\bf Datasets\quad} The aYahoo dataset~\citep{farhadi2009describing} consists of 1850 training images and 794 test images, (ii) aPascal dataset~\citep{farhadi2009describing} has 2869 training images and 2227 test images, (iii) AwA2~\citep{xian2018zero} we used 26125 training images and 11197 test images, (iv) CUB-200-2011~\citep{WahCUB_200_2011} contains 5994 training images and 5794 test images. We used the standard train/test splits for each one of the datasets, except AwA2, where we  follow a 0.7/0.3 random split because there is no standard partition for this dataset.

{\bf Runtime and Infrastructure\quad} The experiments were run on three GeForce RTX 3090 GPUs. Our method completes an epoch every 7 minutes on aYahoo, 10 minutes on aPascal, 1 hours on AwA2 and 1 hours on CUB-200-2011. On average our method runs for approximately 12 epochs on both aYahoo and aPascal, 2-4 epochs on AwA2 and 35-40 on CUB-200-2011. 

\section{Additional Experiments}

In Fig.~\ref{fig:reg} we plot the results of the same experiment as in Fig.~4 in the main text for the CUB-200-2011~\citep{WahCUB_200_2011} dataset. These results conform with the observations summarized in the main text. 

\begin{figure}[t]
\centering
\begin{tabular}{ 
c @{\hspace{0.2\tabcolsep}}
c @{\hspace{0.2\tabcolsep}}
}
\includegraphics[scale=0.255]{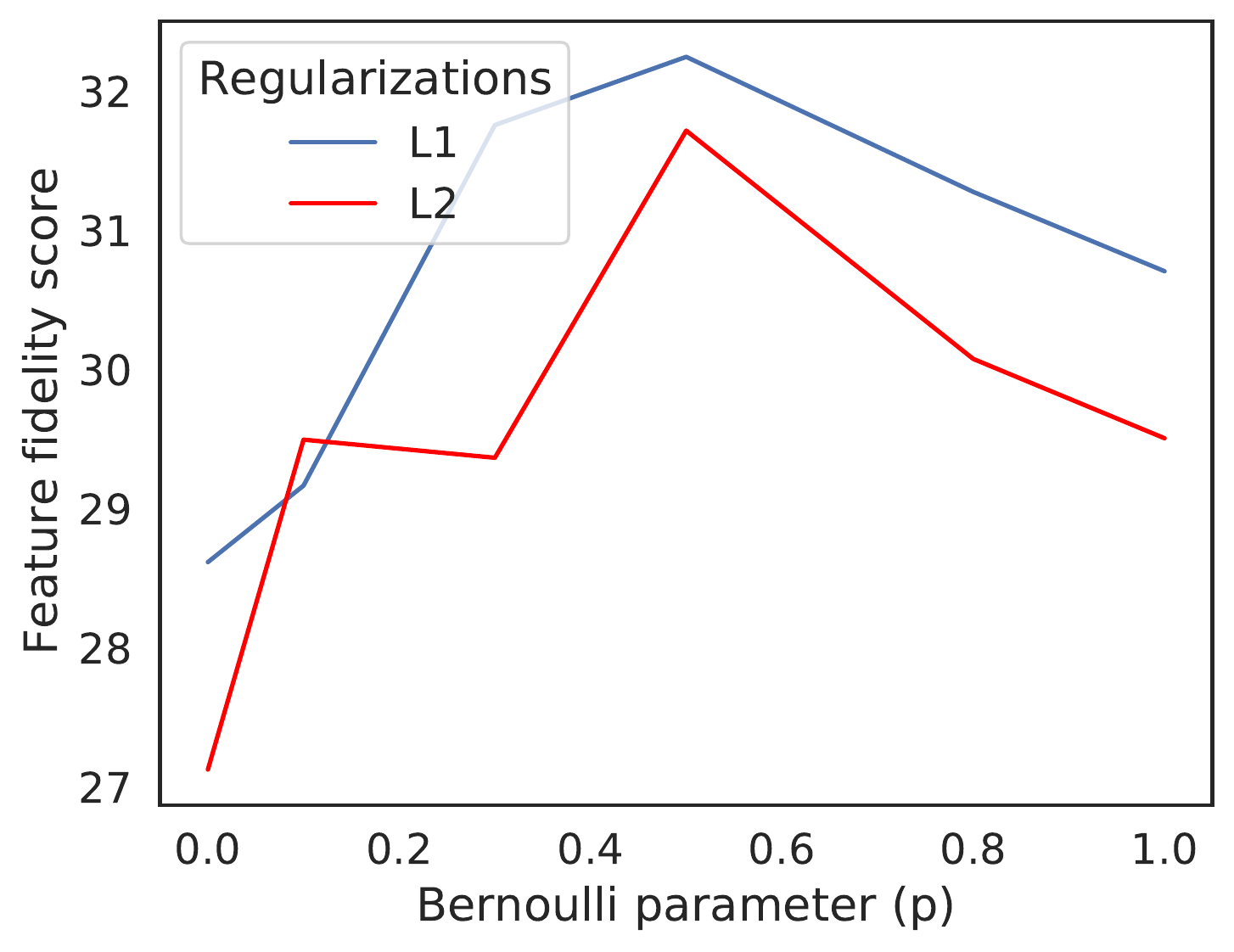}
& \includegraphics[scale=0.255]{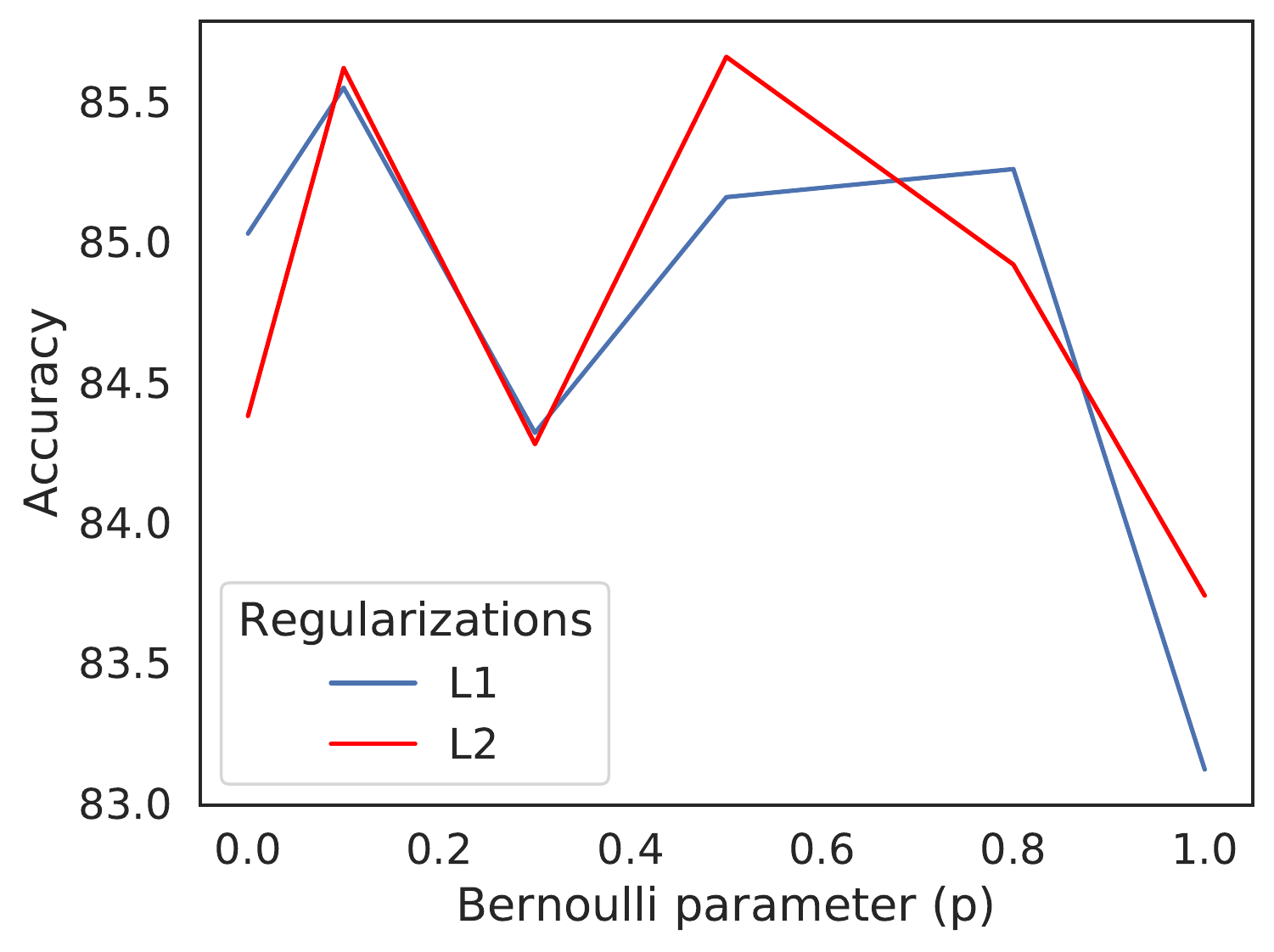} \\
(a) & (b)  \\
\end{tabular}
\caption{{Comparing $d_D(F)$ {\bf(a)} and the accuracy rate {\bf(b)} when varying $p$ for both $L_1$ and $L_2$ regularizations.}}
    \label{fig:reg}
\end{figure}

\section{Proofs of Props.~\ref{lem:1} and~\ref{lem:2}}

\begin{prop}\label{lem:1}
Assume that $\mathcal{Q}(\theta,\omega)$ is convex and $\beta$-smooth w.r.t $\theta$ for any fixed value of $\omega$. Let $\theta_1$ be some initialization and $\omega_1  \in \arg\min_{\omega} \mathcal{Q}(\theta_1,\omega)$. We define $\theta_t$ to be the weights produced after $t$ iterations of applying Gradient Descent on $\mathcal{Q}(\theta,\omega_{t-1})$ over $\theta$ with learning rate $\mu < \beta^{-1}$ and $\omega_t = \arg\min_{\omega} \mathcal{Q}(\theta_{t-1},\omega)$. Then,  we have:
\begin{equation*}
\begin{aligned}
\lim_{t \to \infty} \mathcal{Q}(\theta_{t},\omega_{t}) 
= \lim_{t \to \infty} \min_{\theta} \mathcal{Q}(\theta,\omega_{t}) = \lim_{t \to \infty} \min_{\omega} \mathcal{Q}(\theta_{t},\omega)
\end{aligned}
\end{equation*}
\end{prop}

\begin{proof}
First, we would like to prove that $\mathcal{Q}(\theta_{t},\omega_{t})$ is monotonically decreasing. We notice that for each $t \in \mathbb{N}$, we have:
\begin{equation}
\mathcal{Q}(\theta_{t},\omega_{t}) \leq \mathcal{Q}(\theta_{t},\omega_{t-1})
\end{equation}
since $\omega_{t}$ is the global minimizer of $\mathcal{Q}(\theta_{t},\omega)$. In addition, by the proof of (cf.~\citet{10.5555/2670022}, Thm.~2.1.14), we have:
\begin{equation}\label{eq:a}
\begin{aligned}
\mathcal{Q}(\theta_{t},\omega_{t-1}) 
\leq& \mathcal{Q}(\theta_{t-1},\omega_{t-1}) \\
&- \eta \| \nabla_{\theta}\mathcal{Q}(\theta_{t},\omega_{t-1})\|^2_2 \\
\leq& \mathcal{Q}(\theta_{t-1},\omega_{t-1}),
\end{aligned}
\end{equation}
where $\eta := \mu (1-0.5\beta \mu) > 0$. Therefore, we conclude that  $\mathcal{Q}(\theta_{t},\omega_{t})$ is indeed monotonically decreasing. Since $\mathcal{Q}$ is non-negative, we conclude that $\mathcal{Q}(\theta_t,\omega_t)$ is a convergent sequence. In particular, $\theta_{t}$ is bounded. Otherwise, $\mathcal{Q}(\theta_{t},\omega_{t}) \geq \mathcal{L}_{\mathcal{S}}[h_{\theta_{t}},y] \to \infty$ in contradiction to the fact that $\mathcal{Q}(\theta_{t},\omega_{t})$ is a convergent sequence. 

Next, by applying the convexity of $\mathcal{Q}(\theta,\omega_{t-1})$ (as a function of $\theta$), we have:
\begin{equation}\label{eq:b}
\begin{aligned}
&\mathcal{Q}(\theta_{t},\omega_{t-1}) - \mathcal{Q}(\theta^*_{t},\omega_{t-1}) \\
\leq& \langle \nabla_{\theta} \mathcal{Q}(\theta,\omega_{t-1}) , \theta_{t} - \theta^*_{t} \rangle \\
\leq& \|\nabla_{\theta}\mathcal{Q}(\theta_{t},\omega_{t-1}) \|_2 \cdot \|\theta_{t} - \theta^*_{t} \|_2, \\
\end{aligned}
\end{equation}
where $\theta^*_{t} = \arg\min_{\theta} \mathcal{Q}(\theta,\omega_{t})$. By combining Eqs.~\eqref{eq:a} and~\eqref{eq:b}, we have:
\begin{equation}
\begin{aligned}
&\mathcal{Q}(\theta_{t},\omega_{t}) \\
\leq& \mathcal{Q}(\theta_{t},\omega_{t-1}) 
\leq \mathcal{Q}(\theta_{t-1},\omega_{t-1})\\ 
&- \frac{\eta}{\|\theta_{t} - \theta^*_{t}\|^2_2} \left( \mathcal{Q}(\theta_{t-1},\omega_{t-1}) - \min_{\theta} \mathcal{Q}(\theta,\omega_{t-1}) \right)^2
\end{aligned}
\end{equation}
In particular, 
\begin{equation}
\begin{aligned}
&\mathcal{Q}(\theta_{t-1},\omega_{t-1}) - \mathcal{Q}(\theta_{t},\omega_{t})\\ 
&\geq \frac{\eta}{\|\theta_{t} - \theta^*_{t}\|^2_2} \left( \mathcal{Q}(\theta_{t-1},\omega_{t-1}) - \min_{\theta} \mathcal{Q}(\theta,\omega_{t-1}) \right)^2
\end{aligned}
\end{equation}
Since the left hand side tends to zero and the right hand side is lower bounded by zero, by the sandwich theorem, the right hand side tends to zero as well. Since both $\theta_{t}$ and $\theta^*_{t}$ are bounded sequences ($\{\arg\min_{\theta} \mathcal{Q}(\theta,\omega) \mid \omega \in \Omega\}$ is well-defined and bounded), we conclude that $\lim_{t \to \infty} \mathcal{Q}(\theta_{t},\omega_{t}) = \lim_{t \to \infty} \min_{\theta} \mathcal{Q}(\theta,\omega_{t})$. We also have: $\mathcal{Q}(\theta_{t},\omega_{t}) = \min_{\omega} \mathcal{Q}(\theta_{t},\omega)$ by definition, and therefore, $\lim_{t \to \infty} \mathcal{Q}(\theta_{t},\omega_{t}) = \lim_{t \to \infty} \min_{\omega} \mathcal{Q}(\theta_{t},\omega)$ as well.
\end{proof}

\begin{prop}\label{lem:2}
Assume that $\mathcal{Q}(\theta,\omega)$ is a twice-continuously differentiable, element-wise convex (i.e., convex w.r.t $\theta$ for any fixed value of $\omega$ and vice versa), Lipschitz continuous and $\beta$-smooth function whose saddle points are strict. Let $\theta_t$, $\omega_t$ be the weights produced after $t$ iterations of applying BCGD on $\mathcal{Q}(\theta,\omega)$ with learning rate $\mu < \beta^{-1}$. Then, $(\theta_t,\omega_t)$ converges to a local minima $(\hat{\theta},\hat{\omega})$ of $\mathcal{Q}$ that is also an equilibrium point.
\end{prop}

\begin{proof}
Firstly, since $\mathcal{Q}(\theta,\omega)$ is a twice-continuously differentiable, Lipschitz continuous and $\beta$-smooth function, by Prop.~3.4 and Cor.~3.1 in~\citep{song2017block}, $(\theta_{t},\omega_{t})$ converge to a local minima $(\hat{\theta},\hat{\omega})$ of $\mathcal{Q}$. Therefore, it is left to show that $(\hat{\theta},\hat{\omega})$ is also an equilibrium point. We note that $\mathcal{Q}(\theta,\omega)$ is element-wise convex and $\beta$-smooth. By the proof of (cf.~\citet{10.5555/2670022}, Thm.~2.1.14), we have:
\begin{equation}\label{eq:c}
\begin{aligned}
&\mathcal{Q}(\theta_{t+1},\omega_{t}) \\
\leq& \mathcal{Q}(\theta_{t},\omega_{t}) - \eta \| \nabla_{\theta}\mathcal{Q}(\theta_{t},\omega_{t})\|^2_2,
\end{aligned}
\end{equation}
where $\eta := \mu (1-0.5\beta \mu) > 0$. By applying the convexity of $\mathcal{Q}(\theta,\omega_{t})$ (as a function of $\theta$), we have:
\begin{equation}\label{eq:d}
\begin{aligned}
&\mathcal{Q}(\theta_{t},\omega_{t}) - \mathcal{Q}(\theta^*_{t},\omega_{t}) \\
\leq& \langle \nabla_{\theta} \mathcal{Q}(\theta,\omega_{t}) , \theta_{t} - \theta^*_{t} \rangle \\
\leq& \|\nabla_{\theta}\mathcal{Q}(\theta_{t},\omega_{t}) \|_2 \cdot \|\theta_{t} - \theta^*_{t} \|_2, \\
\end{aligned}
\end{equation}
where $\theta^*_{t} = \arg\min_{\theta} \mathcal{Q}(\theta,\omega_{t})$ and $\omega^*_{t} = \arg\min_{\omega} \mathcal{Q}(\theta_{t+1},\omega)$. By combining Eqs.~\eqref{eq:c} and~\eqref{eq:d}, we have:
\begin{equation}
\begin{aligned}
&\mathcal{Q}(\theta_{t+1},\omega_{t}) \leq \mathcal{Q}(\theta_{t},\omega_{t}) \\
&- \frac{\eta}{\|\theta_{t} - \theta^*_{t}\|^2_2} \left( \mathcal{Q}(\theta_{t},\omega_{t}) - \min_{\theta} \mathcal{Q}(\theta,\omega_{t}) \right)^2
\end{aligned}
\end{equation}
and similarly, we have:
\begin{equation}
\begin{aligned}
&\mathcal{Q}(\theta_{t+1},\omega_{t+1}) \leq \mathcal{Q}(\theta_{t+1},\omega_{t})  \\
&- \frac{\eta}{\|\omega_{t} - \omega^*_{t}\|^2_2} \left( \mathcal{Q}(\theta_{t+1},\omega_{t}) - \min_{\omega} \mathcal{Q}(\theta_{t+1},\omega) \right)^2
\end{aligned}
\end{equation}

In particular, we have:
\begin{equation}\label{eq:beta}
\begin{aligned}
&\mathcal{Q}(\theta_{t+1},\omega_{t+1}) \\
\leq& \mathcal{Q}(\theta_{t},\omega_{t}) \\
&- \frac{\eta}{\|\theta_{t} - \theta^*_{t}\|^2_2} \left( \mathcal{Q}(\theta_{t},\omega_{t}) - \min_{\theta} \mathcal{Q}(\theta,\omega_{t}) \right)^2 \\
&- \frac{\eta}{\|\omega_{t} - \omega^*_{t}\|^2_2} \left( \mathcal{Q}(\theta_{t+1},\omega_{t}) - \min_{\omega} \mathcal{Q}(\theta_{t+1},\omega) \right)^2
\end{aligned}
\end{equation}
Therefore, the sequence $\mathcal{Q}(\theta_{t},\omega_{t})$ is monotonically decreasing. Since $\mathcal{Q}(\theta_{t},\omega_{t})$ is non-negative, it converges to some non-negative constant. By Eq.~\eqref{eq:beta}, we have:
\begin{equation}
\begin{aligned}
&\mathcal{Q}(\theta_{t+1},\omega_{t+1}) - \mathcal{Q}(\theta_{t},\omega_{t})\\
\geq& \frac{\eta}{\|\theta_{t} - \theta^*_{t}\|^2_2} \left( \mathcal{Q}(\theta_{t},\omega_{t}) - \min_{\theta} \mathcal{Q}(\theta,\omega_{t}) \right)^2 \\
&+ \frac{\eta}{\|\omega_{t} - \omega^*_{t}\|^2_2} \left( \mathcal{Q}(\theta_{t+1},\omega_{t}) - \min_{\omega} \mathcal{Q}(\theta_{t+1},\omega) \right)^2
\end{aligned}
\end{equation}
Since the left hand side tends to zero and the right hand side is lower bounded by zero, by the sandwich theorem, the sequences $\frac{\eta}{ \|\theta_{t} - \theta^*_{t}\|^2_2} \left( \mathcal{Q}(\theta_{t},\omega_{t}) - \min_{\theta} \mathcal{Q}(\theta,\omega_{t}) \right)^2 $ and $ \frac{\eta}{\|\omega_{t} - \omega^*_{t}\|^2_2} \left( \mathcal{Q}(\theta_{t+1},\omega_{t}) - \min_{\omega} \mathcal{Q}(\theta_{t+1},\omega) \right)^2$ tend to zero. We note that $\theta_{t}$ and $\omega_{t}$ are convergent sequences, and therefore, are bounded as well. In addition, we recall that the sets $\{\arg\min_{\theta} \mathcal{Q}(\theta,\omega) \mid \omega \in \Omega\}$ and $\{\arg\min_{\omega} \mathcal{Q}(\theta,\omega) \mid \theta \in \Theta\}$ are well-defined and bounded. Hence, the terms $\|\theta_{t} - \theta^*_{t}\|^2_2$ and $\|\omega_{t} - \omega^*_{t}\|^2_2$ are bounded. Thus, we conclude that the sequences $ \left( \mathcal{Q}(\theta_{t},\omega_{t}) - \min_{\theta} \mathcal{Q}(\theta,\omega_{t}) \right)^2 $ and $\left( \mathcal{Q}(\theta_{t+1},\omega_{t}) - \min_{\omega} \mathcal{Q}(\theta_{t+1},\omega) \right)^2$ tend to zero. In particular,
\begin{equation}
\begin{aligned}
\lim_{t \to \infty} \mathcal{Q}(\theta_{t},\omega_{t}) 
&= \lim_{t \to \infty} \min_{\theta} \mathcal{Q}(\theta,\omega_{t}) \\
\lim_{t \to \infty} \mathcal{Q}(\theta_{t+1},\omega_{t})  &= \lim_{t \to \infty} \min_{\omega} \mathcal{Q}(\theta_{t},\omega)
\end{aligned}
\end{equation}
Since $\mathcal{Q}$ is a Lipschitz continuous function, we have: (i) $\lim_{t \to \infty} \mathcal{Q}(\theta_{t+1},\omega_{t}) = \lim_{t \to \infty} \mathcal{Q}(\theta_{t},\omega_{t}) = \mathcal{Q}(\hat{\theta},\hat{\omega})$, (ii) $\lim_{t \to \infty} \min_{\theta} \mathcal{Q}(\theta,\omega_{t}) =  \min_{\theta} \mathcal{Q}(\theta,\hat{\omega})$ and (iii) $\lim_{t \to \infty} \min_{\omega} \mathcal{Q}(\theta_{t},\omega) =  \min_{\omega} \mathcal{Q}(\hat{\theta},\omega)$. Therefore, we finally conclude that $(\hat{\theta},\hat{\omega})$ is an equilibrium point of $\mathcal{Q}$.
\end{proof}

\bibliography{ref}
\bibliographystyle{icml_format/icml2021.bst}